\documentclass[11pt]{article}

\usepackage{bbm}
\usepackage{amsmath,amsfonts,amssymb,enumerate}
\usepackage{color}
\usepackage{lscape}
\usepackage{latexsym}
\usepackage{multirow}
\usepackage{rotating}
\usepackage{threeparttable}
\usepackage{graphicx}
\usepackage{algpseudocode,algorithm}
\usepackage[margin=1in]{geometry}

\numberwithin{equation}{section}
\newtheorem{theorem}{\bf Theorem}[section]
\newtheorem{lemma}[theorem]{\bf Lemma}

\def\HH{\mathbb{H}}  

\def \bR {\Bbb R}
\def \bN {\Bbb N}
 
\def\DD{\mathbb{D}} 

\def\cA{{\cal{A}}}

\def\cD{{\cal{D}}}

\def\cG{{\cal{G}}}

\def\cI{{\cal{I}}}

\def\cN{{\cal{N}}}
\def\cP{{\cal{P}}}

\def\bbb{{\bf{b}}}
\def\bc{{\bf{c}}}

\def\ff{{\bf{f}}}
\def\bg{{\bf{g}}}

\def\be{{\bf{e}}}
\def\bx{{\bf{x}}}
\def\by{{\bf{y}}}

\def\b0{{\bf{0}}}

\def\bP{{\bf{P}}}

\def\bW{{\bf{W}}}

\newenvironment{proof}{\noindent{\em Proof:}}{\quad \hfill$\Box$\vspace{2ex}}

\graphicspath{{figs/}}

\begin{document}



\title{\bf Successive Affine Learning for Deep Neural Networks}
\author{Yuesheng Xu\thanks{Department of Mathematics and Statistics, Old Dominion University, Norfolk, VA 23529, USA. E-mail address: {\it y1xu@odu.edu}. }}
\date{}
\maketitle

\begin{abstract}

This paper introduces a successive affine learning (SAL) model for constructing deep neural networks (DNNs). Traditionally, a DNN is built by solving a non-convex optimization problem. It  is often challenging to solve such a problem numerically due to its non-convexity and having a large number of layers. To address this challenge, inspired by the human education system,  the multi-grade deep learning (MGDL) model was recently initiated by the author of this paper. The MGDL model learns a DNN in several grades, in each of which one constructs a shallow DNN consisting of a relatively small number of layers. The MGDL model still requires solving several non-convex optimization problems. The proposed SAL model mutates from the MGDL model. 
Noting that each layer of a DNN consists of an affine map followed by an activation function, we propose to learn the affine map by solving a quadratic/convex optimization problem which involves the activation function only {\it after} the weight matrix and the bias vector for the current layer have been trained. In the context of function approximation, for a given function the SAL model generates an expansion of the function with adaptive basis functions in the form of DNNs. We establish the Pythagorean identity and the Parseval identity for the system generated by the SAL model. Moreover, we provide a convergence theorem of the SAL process in the sense that either it terminates after a finite number of grades or the norms of its optimal error functions strictly decrease to a limit
as the grade number increases to infinity.
Furthermore, we present numerical examples of proof of concept which demonstrate that the proposed SAL model significantly outperforms the traditional deep learning model.
\end{abstract}


\section{Introduction}
The goal of this paper is to introduce a successive affine learning (SAL) model for the construction of deep neural
networks (DNNs) for deep learning.
The great success of deep learning \cite{Goodfellow, LeCun} and its impact to science, technology and our society have been widely recognized \cite{DFVX, Douglas2023, Ida2019, Jiang, Krizhevsky, Raissi, DShen, Torlai, XuZeng2022}. Especially, the recently launched ChatGPT, based on the generative pre-trained transformer, has garnered attention for its detailed responses and articulate answers across many domains of knowledge \cite{Lock2022}.
The core of deep learning is to construct a deep neural network (DNN) as a prediction, decision function, and its successes are, to a great extent,  due to the mighty expressiveness of DNNs in representing a function \cite{Daubechies2022, Du2019, Poggio2017, Shen2, XuZhang2021, XuZhang2022}. 

In deep learning, a DNN is learned by solving an optimization problem which determines its parameters (weight matrices and bios vectors) that define it with an activation function. The optimization problem that learns a DNN is highly non-convex and has a large number of layers. Solving such an optimization problem  has been recognized as a major computational obstacle of deep learning. 
A commonly used method to solve the optimization problem is the stochastic gradient descent method \cite{Bottou1998, Bottou2012, Kingma}, with a choice of the initial guess proposed in \cite{He2015}. However, gradient-based optimization starting from random initialization appears to often get stuck in poor solutions \cite{Bengio}. Noting that the existing deep learning model uses a {\it single} optimization problem to train all layers of the DNN at one time, the more layers the DNN possesses, the severer the  non-convexity the resulting optimization problem is, and thus, the more difficulty one would encounter when trying to solve it. 


Inspired by human learning process which is often organized in grades, the multi-grade deep learning (MGDL) model was recently proposed in \cite{Xu2023} by the author of this paper, where DNNs were learned grade-by-grade. Instead of solving one {\it  single} optimization problem with a large number of layers, with the MGDL model we solve several optimization problems, each with a relatively small number of layers which determine a shallow neural network for a grade. The outcome of the MGDL model is a DNN with a structure different from the one learned by the single-grade learning but with a comparable approximation accuracy. Often, it is easier to learn several shallow neural networks than a deep one. The MGDL model reduces the complexity and alleviates the difficulty, of learning DNNs by the single-grade learning model. 

The current paper continues the general theme of \cite{Xu2023}, with  bold advancements. In the SAL model to be proposed, every grade contains only one layer and {\it free} the activation function from the associated optimization problem for training the weight matrix and the bias vector of the layer, so that the resulting optimization problem to learn them becomes either quadratic or convex.  
The development of this model is inspirited by an ancient philosophical principle: ``One step at a time leads to thousands of miles'' (Xun Zi, 313 - 238 B.C., an ancient Chinese great thinker);  ``the great doesn’t happen through impulse alone, and is a succession of little things that are brought together''
(Vincent van Gogh).
At each of the small steps, we solve a quadratic/convex optimization problem for one layer and by accumulating many of such steps we end up building a DNN of many layers, which has excellent functional expressiveness. 
In particular, in the context of function approximation, we identify the convex optimization problem of a grade as the orthogonal projection of the error function of the previous grade onto a linear subspace determined by the neural network learned from the previous grades. This observation leads to establishment of theoretical justifications of the SAL model. 

A DNN learned by the SAL model is the superposition of all the neural networks learned in all grades. Each term of the superposition is the term learned in the previous grade composed with a new layer whose weight matrix and bias vector are learned in the current grade from the error function of the previous grade by a convex/quadratic optimization problem. The design of the SAL model takes the advantage of the structure of a layer: Each layer of a DNN consists of an affine map followed by an activation function. We then propose to learn the affine map, defined by the weight matrix and the bias vector, by solving a quadratic/convex optimization problem without involving the activation function of the present layer. Only after the weight matrix and the bias vector of the layer have been obtained, we apply the activation function of the layer. In this way, the resulting optimization problem for each layer is convex/quadratic.

The innovation of the SAL model lies on avoiding solving a non-convex optimization problem, instead solving only convex/quadratic optimization problems to learn affine maps. In this way, standard numerical methods such as the Nesterov algorithm \cite{Nesterov}, the conjugate gradient method and the preconditioned conjugate gradient method \cite{Golub} are applicable for solving the convex/quadratic optimization problems, leading to a more accurate, effective and efficient learning model, because these numerical optimization methods are all easy to implement. In particular, the SAL model overcomes the vanishing gradient issue from which training a standard DNN normally surfers. Moreover, the SAL model is particularly suitable for adaptive approximation. It is convenient to add a new grade to the neural network learned from the previous grades. More importantly, we establish rigorous mathematical foundation for functions generated by the SAL model. This makes the SAL model a practical useful tool with sound mathematical foundation, unlike the traditional DNN model which it is challenging to implement and is often a black-box in terms of mathematical analysis.
The theoretical results presented in this paper for the SAL model sheds light on ``harmonic analysis'' of
DNNs.


We organize this paper in nine sections. In section 2, we review the traditional single-grade learning and the recently proposed MGDL model for building DNNs. 
Section 3 describes the SAL model for both function approximation and data fitting. For simplicity of presentation, we present the basic idea in the special case when the weight matrices are square matrices. We discuss in section 4 the SAL model with the average pooling which allows the weight matrices to be non-square in order to increase the expressiveness of the resulting DNNs by increasing the number of neurons in a layer. Section 5 is devoted to theoretical analysis of the SAL model with the average pooling. We show that the DNN learned by the SAL model enjoys the nice properties such as the Pythagorean identity and the Parseval identity. In section 6, we address the smoothing issue related to the SAL model. We discuss in section 7 crucial issues related to implementation of the proposed SAL model. In section 8, we provide two proof of concept numerical examples. Finally, we make conclusive remarks in section 9.

\section{Deep Neural Networks: Single-Grade Learning vs Multi-Grade Learning}

In this section, we recall the definition of the standard deep learning model - the single-grade learning model, and review the multi-grade deep learning model introduced recently in \cite{Xu2023} by the author of this paper. 

A DNN is  a function $\ff:\bR^s\to\bR^t$ formed by compositions of vector-valued functions, each of which is defined by an activation function applied to an affine map, where $s$ and $t$ are positive integers.
Given a univariate function $\sigma: \bR\to\bR$,
a vector-valued function may be defined for $\bx:=[x_1, x_2,\dots, x_d]^\top\in\bR^d$ by 
\begin{equation}\label{activationF}
\sigma(\bx):=[\sigma(x_1),\dots,\sigma(x_d)]^\top.
\end{equation}
It is convenient to use compact notation for compositions of functions.
For $n$ vector-valued functions $f_k$, $k\in\bN_n$, where the range of $f_k$ is contained in the domain of $f_{k+1}$, for $k\in\bN_{n-1}$, we denote the consecutive composition of $f_k$, $k\in\bN_n$, by
\begin{equation}\label{consecutive_composition}
    \bigodot_{k=1}^n f_k:=f_n\circ f_{n-1}\circ\cdots\circ f_2\circ f_1,
\end{equation}
whose domain is that of $f_1$. 
Let $m_0:=s$ and $m_n:=t$.
Given $\bW_i\in\bR^{m_i\times m_{i-1}}$ and $\bbb_i\in\bR^{m_i}$, $i\in\bN_n$, 
a DNN is a function defined by
\begin{equation}\label{DNN}
    \cN_n(\bx):=\left(\bW_n\bigodot_{i=1}^{n-1} \sigma(\bW_i \cdot+\bbb_i)+\bbb_n\right)(\bx),\ \ \bx\in\bR^s.
\end{equation}
The $n$-th layer is the output layer.
Note that for each $i\in\bN_n$, $\bW_i \cdot+\bbb_i$ is an affine map.
From \eqref{DNN} and the definition \eqref{activationF}, a DNN can be defined recursively by
\begin{equation}\label{Step1}
    \cN_1(\bx):=\sigma(\bW_1 \bx+\bbb_1)
\end{equation}
and
\begin{equation}\label{Recursion}
    \cN_{k+1}(\bx)=\sigma(\bW_{k+1}\cN_k(\bx)+\bbb_{k+1}), \ \ \bx\in \bR^s, \ \ \mbox{for all} \ \ k\in \bN_{n-1},
\end{equation}
where for $k:=n-1$, $\sigma$ in \eqref{Recursion} is the identity map. Clearly, from \eqref{Step1} and \eqref{Recursion}, we observe that each layer of a DNN consists of an affine map followed by an activation function.

A DNN may be learned from given data. From $m$ pairs of given points $(\bx_i, \by_i)$, $i\in\bN_m:=\{1,2,\dots,m\}$, with $\bx_i\in\bR^s$ and $\by_i\in \bR^t$, one may learn a function $\ff:\bR^s\to\bR^t$, in a form of a DNN of $n$ layers composed of $n-1$ hidden layers and one output layer by determining $n$ weight matrices $\bW_k$ and bias vectors $\bbb_k$, $k\in\bN_{n}$, through one or more activation functions. 
Specifically,
one can learn a function 
\begin{equation}\label{TraditionalDNN}
  \cN_n(\bx):= \cN_n(\{\bW_j^*,\bbb_j^*\}_{j=1}^n;\bx),\ \ \bx\in\bR^s
\end{equation}
with the parameters given by 
\begin{align}\label{Basic-Min-Problem}
    \{\bW_j^*, \bbb_j^*\}_{j=1}^n
    :=    {\rm argmin}\left\{\sum_{k=1}^m\|\cN_n(\{\bW_j,\bbb_j\}_{j=1}^n; \bx_k)-\by_k\|_{\ell_2}^2: \bW_j\in\bR^{m_j\times m_{j-1}}, \bbb_j\in\bR^{m_j}, j\in\bN_{n}\right\},
\end{align}
where $\|\cdot\|_{\ell_2}$ denotes the Euclidean vector norm of $\bR^t$. 

The continuous version of the learning problem \eqref{Basic-Min-Problem} in the context of function approximation may be described as follows. Suppose that $\DD\subseteq\bR^s$ is a domain, and let $L_2(\DD)$ denote the usual Hilbert space of the square-integrable functions $g$ on $\DD$ with
$$
\|g\|_2:=\left(\int_\DD|g(\bx)|^2d\bx\right)^\frac12<+\infty.
$$
By $L_2(\DD,\bR^t)$ we denote the Hilbert space of the vector-valued functions $\bg:=[g_1, g_2,\dots, g_t]^\top: \DD\to\bR^t$ with $g_j\in L_2(\DD)$, $j\in\bN_t$. 
The inner-product and the norm of the space $L_2(\DD,\bR^t)$ are defined respectively, for $\ff,\bg\in L_2(\DD,\bR^t)$ by
$$
\left<\ff,\bg\right>:=\sum_{j=1}^t\int_\DD f_j(\bx)g_j(\bx)d\bx
$$
and 
$$
\|\bg\|:=\left[\sum_{j=1}^t \|g_j\|_2^2\right]^\frac{1}{2}.
$$
Given a function $\ff\in L_2(\DD,\bR^t)$, we wish to learn a DNN $\cN_n$ in the form of \eqref{TraditionalDNN} with the parameters given by
\begin{equation}\label{Basic-Min-Problem-cont}
\{\bW_j^*,\bbb_j^*\}_{j=1}^n:=    {\rm argmin}\left\{\|\ff(\cdot)-\cN_n(\{\bW_j,\bbb_j\}_{j=1}^n;\cdot)\|^2: \bW_j\in\bR^{m_j\times m_{j-1}}, \bbb_j\in\bR^{m_j}, j\in\bN_{n}\right\}.
\end{equation}
Clearly, the function $\cN_n(\{\bW_j,\bbb_j\}_{j=1}^n;\cdot)$ is a best approximation to the given function $\ff$ from the non-convex set $\Omega_n$ of DNNs having the form \eqref{DNN}. 

Learning a DNN from either discrete data  or a continuous function requires to solve minimization problem \eqref{Basic-Min-Problem} or \eqref{Basic-Min-Problem-cont}. Both minimization problems \eqref{Basic-Min-Problem} and \eqref{Basic-Min-Problem-cont} are {\it single-grade} learning models. Such a learning model learns all weight matrices and bias vectors by solving a single optimization problem, which is a highly non-convex problem and is challenging to solve. A multi-grade deep learning model was recently put forward in \cite{Xu2023} to alleviate the difficulty in learning all parameters of a single-grade deep learning model.

We now recall the $l$-grade learning model introduced in \cite{Xu2023} for learning DNNs from a continuous function $\ff\in L_2(\DD,\bR^t)$, for $l\in \bN_n$. We choose $k_j\in \bN$, for $j\in\bN_l$, so that $\sum_{j=1}^lk_j=n-1$, and for each $k_j$, we choose a set of matrix widths $\{m_k: k=0,1,\dots,k_j\}$, which may be different for different $k_j$, and $m_{k_j}=t$. Note that each integer $k_j$ is relatively small in comparison to $n$.
The first grade is to learn the neural network $\cN_{k_1}$ having the form of \eqref{DNN} with $n:=k_1$. Specifically, we define the grade 1 error function by
\begin{equation} \label{error1-G}
\be_1(\{\bW_j,\bbb_j\}_{j=1}^{k_1};\bx):=\ff(\bx)-\cN_{k_1}(\{\bW_j,\bbb_j\}_{j=1}^{k_1};\bx), \ \ \bx\in\bR^s,
\end{equation}
where $\{\bW_j,\bbb_j\}_{j=1}^{k_1}$ are parameters to be learned. Letting $m_0:=s$,  we solve the optimization problem
\begin{equation}\label{min1-G}
\min\{\|\be_1(\{\bW_j,\bbb_j\}_{j=1}^{k_1};\cdot)\|^2: \bW_j\in\bR^{m_j\times m_{j-1}}, \bbb_j\in\bR^{m_j}, j\in\bN_{k_1}\},
\end{equation}
for 
$\{\bW_{1,j}^*,\bbb_{1,j}^*\}_{j=1}^{k_1}$, which gives the approximation of grade 1
$$
\ff_1(\bx)=\cN^*_{k_1}(\bx):=\cN_{k_1}(\{\bW_j^*,\bbb_j^*\}_{j=1}^{k_1};\bx),\ \ \bx\in\bR^s.
$$
We then define the optimal error of grade 1 by setting
$$
\be^*_1(\bx):=\ff(\bx)-\ff_1(\bx),\ \ \mbox{for}\ \ \bx\in\bR^s,
$$
from which an approximation of grade 2 is to be learned.

Assume that for $i\geq 1$, the neural networks $\cN^*_{k_i}$ of grades $i$,
have been learned with the optimal error $\be_{i}^*$.
We then define the error function of grade $i+1$ by
$$
\be_{i+1}(\{\bW_j, \bbb_j\}_{j=1}^{k_{i+1}};\bx):=\be_{i}^*(\bx)-(\cN_{k_{i+1}}(\{\bW_j, \bbb_j\}_{j=1}^{k_{i+1}};\cdot)\circ\cN^*_{k_{i}}\circ\cdots\circ\cN^*_{k_1})(\bx),\ \ \bx\in\bR^s,
$$
where $\cN_{k_{i+1}}$ is a neural network having the form \eqref{DNN} with $n:=k_{i+1}$ to be learned in grade $i+1$.
Let $m_0:=t$ and $m_{k_{i+1}}:=t$, and we solve the optimization problem
\begin{equation}\label{minl-G}
\min\{\|\be_{i+1}(\{\bW_{j}, \bbb_{j}\}_{j=1}^{k_{i+1}};\cdot)\|^2: \bW_j\in\bR^{m_j\times m_{j-1}}, \bbb_j\in\bR^{m_j},  j\in\bN_{k_{i+1}}\},
\end{equation}
to find the optimal parameters $\{\bW_{i+1,j}^*, \bbb_{i+1,j}^*\}_{j=1}^{k_{i+1}}$. When solving the optimization problem \eqref{minl-G}, the weight matrices and bias vectors of the neural networks $\cN^*_{k_{1}}, \dots, \cN^*_{k_{i}}$ are all fixed. The optimal parameters $\{\bW_{i+1,j}^*, \bbb_{i+1,j}^*\}_{j=1}^{k_{i+1}}$ define the neural network 
$$
\cN^*_{k_{i+1}}:=\cN_{k_{i+1}}(\{\bW_{i+1,j}^*, \bbb_{i+1,j}^*\}_{j=1}^{k_{i+1}};\cdot)
$$ 
and give the approximation of grade $i+1$
$$
\ff_{i+1}(\bx):=(\cN^*_{k_{i+1}}\circ\cN^*_{k_{i}}\circ\cdots\circ\cN^*_{k_1})(\bx), \ \ \bx\in\bR^s.
$$
We then define the optimal error of grade $i+1$ by
$$
\be^*_{i+1}(\bx):=\be^*_{i}(\bx)-\ff_{i+1}(\bx), \ \ \mbox{for}\ \ \bx\in\bR^s.
$$
Note that $\ff_{i+1}$, the newly learned neural network $\cN_{k_{i+1}}^*$ stacked on the top of the neural network $\cN_{k_{i}}^*\circ\cdots\circ\cN_{k_1}^*$
learned in the previous grades, is a best approximation from the set 
\begin{equation}\label{Def-Omega(i+1)}
    \Omega_{i+1}:=\{\cN_{k_{i+1}}(\{\bW_j, \bbb_j\}_{j=1}^{k_{i+1}};\cdot)\circ\cN^*_{k_{i}}\circ\cdots\circ\cN^*_{k_1}: \bW_j\in\bR^{m_j\times m_{j-1}}, \bbb_j\in\bR^{m_j},  j\in\bN_{k_{i+1}}\}
\end{equation}
to $\be^*_i$. The $l$-grade learning model generates the neural network
\begin{equation}\label{Final-NN}
    \overline{\ff}_l:=\sum_{i=1}^l\ff_i,
\end{equation}
which is the superposition of all $l$ networks $\ff_i$, $i\in\bN_l$, unlike the neural network $\cN_n$ learned by \eqref{Basic-Min-Problem-cont}. In each grade, $\ff_i$ is a shallow network learned in grade $i$ composed with the shallow networks learned from the previous grades. In general, the neural network $\overline{\ff}_l$ has a stairs-shape. 

Unlike the traditional single-grade deep learning model,  which solves one optimization problem \eqref{Basic-Min-Problem} of $n$ layers, the $l$-grade model solves $l$ optimization problems
\eqref{min1-G}
and \eqref{minl-G}. Since integers $k_j$ are significantly smaller than $n$, the MGDL model can alleviate the computational challenges, such as being stuck at a local minimizer and vanishing gradient issue. However, the MGDL model still requires to solve non-convex optimization problems. It is highly desirable to develop a model, with sound mathematical foundation, which has the excellent expressiveness of DNNs, while escaping from the troublesome training process of the traditional deep learning model caused by its non-convexity. 
Can one design a special MGDL model which solves {\it only} convex optimization problems? It is the goal of this paper to answer this question. 

The SAL model to be proposed mutates from the MGDL model described above with specializing to the case in which each grade consists of only one layer whose weight matrix and bias vector are found by solving a quadratic/convex optimization problem {\it before} involving the activation function of the layer. In each grade, we learn an affine map for the grade. The SAL model has a multi-grade learning nature with avoiding solving non-convex optimization problems for its grades. We will develop the SAL model in the next several sections.


\section{Successive Affine Learning Model}\label{Linear}

In this section, we describe the SAL model, a mutated MGDL model via successively learning affine maps. The proposed model constructs a deep neural network {\it without} solving a non-convex optimization problem. 

Having a close examination of the structure of a DNN, one can see that it has the following architecture: Each layer of a neural network consists of an affine map (a weight matrix and a bias vector) followed by neurons (compositions with the activation function). Figure \ref{SAL} illustrates the  architecture of a neural network, where the rectangles represent affine maps and the circles represent neurons.  When focusing only on one layer with all parameters of the previous layers {\it fixed}, determining the affine map of the current layer, before applying the activation function, is a quadratic/convex optimization problem, since the activation function of the current layer is not involved in training of the affine map. The activation function applied after the training of the affine map of the current layer will play a role for training of the affine maps of the following layers. Based on this insight of neural networks, we propose the SAL model.

\begin{figure}[H]
\centering
\includegraphics[width=0.8\textwidth, height=0.15\textwidth]{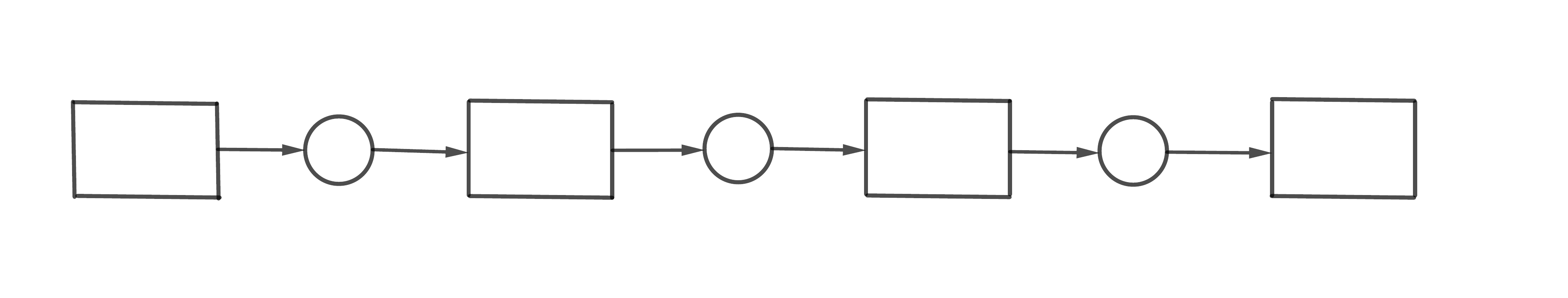}
\caption{Architecture of a neural network}\label{SAL}
\end{figure}

We first describe the SAL model for learning  a function in $L_2(\DD,\bR^t)$.
Given a vector-valued function $\ff\in L_2(\DD,\bR^t)$, we wish to learn a deep neural network that represents the function. 
To free ourselves from the tedious technical details so that we can focus on the main idea and big picture, in this section we confine ourselves to the case that weight matrices are square and postpone the more general and more realistic case until the next section. Also, in the description to follow, we choose $\DD:=\bR^s$.

We now describe the SAL model which builds a neural network that approximates the given function $\ff$. 
As we pointed out earlier, the SAL model mutates from a special case of the MGDL model where each grade consists of only one layer. We first outline learning in grade 1. 
For matrix $\bW\in \bR^{t\times s}$ and vector $\bbb\in \bR^t$, we define the initial error function by
\begin{equation}\label{error-function1}
    \be_1(\bW,\bbb; \bx):=\ff(\bx)-(\bW\bx+\bbb), \ \ \bx\in\bR^s.
\end{equation}
This differs from the error function of grade 1 for a MGDL model whose grades contain only one layer. In the present case, the definition of the error function does not involve the activation function.
We then find 
\begin{equation}\label{min1}
(\bW_1^*,\bbb_1^*):=    {\rm argmin}\{\|\be_1(\bW,\bbb; \cdot)\|^2: \bW\in  \bR^{t\times s}, \bbb\in \bR^t\}.
\end{equation}
Note that \eqref{min1} is a quadratic optimization problem with respect to $\bW$ and $\bbb$, and thus, it can be efficiently solved by various existing algorithms such as the gradient descent method, the Nesterov algorithm, the conjugate gradient method and
the preconditioned conjugate gradient method.
With $\bW_1^*$ and $\bbb_1^*$ found, we obtain the affine map  (linear function)
\begin{equation}\label{ff1}
    \ff_1(\bx):=\bW^*_1\bx+\bbb_1^*, \ \ \bx\in\bR^s,
\end{equation}
that approximates $\ff$ 
with the optimal initial error $\be_1^*$ given by 
\begin{equation}\label{Op-Error-Grade1}
    \be_1^*(\bx):=\be_1(\bW_1^*,\bbb_1^*; \bx), \ \ \bx\in\bR^s.
\end{equation}
It follows from definitions \eqref{error-function1}, \eqref{ff1} and \eqref{Op-Error-Grade1}  that 
\begin{equation}\label{op-error-function1}
\be_1^*(\bx)=\ff(\bx)-\ff_1(\bx), \ \ \bx\in\bR^s.
\end{equation}
We then define the neural network of grade 1 by
\begin{equation}\label{initial-network}
    \cN_1(\bx):=\sigma(\ff_1(\bx)), \ \ \bx\in\bR^s.
\end{equation}
Notice that the vector-valued function $\cN_1: \bR^s\to\bR^t$ contains neurons of the initial layer.
Clearly, from \eqref{op-error-function1} we note that $\be_1^*$ is not the error between $\ff$ and the initial network $\cN_1$, but rather the error of the linear function approximation $\ff_1$ of $\ff$. This is because we find $\bW_1^*$ and $\bbb_1^*$ before applying the activation function. Although the activation function $\sigma$ is not involved in training the weight matrix $\bW^*_1$ and the bias vector $\bbb_1^*$, it will play a role in learning of grades that follow. 
Usually, $\|\be_1^*\|$ is not small, which means that $\be_1^*$ contains useful information of the original function $\ff$. Hence, learning the neural network of grade 2 is required.
The introduction of $\cN_1$ in \eqref{initial-network} is to prepare for moving up to learning of higher grades.

We next describe the SAL model of grade $k$ for $k\geq 2$.
Suppose that the neural networks $\ff_{k-1}$, $\cN_{k-1}$ and the optimal error $\be_{k-1}^*$ of grade $k-1$ 
have been constructed. 
For  matrix $\bW\in \bR^{t\times t}$ and  vector $\bbb\in \bR^t$, we
define the error function of grade $k$ by
\begin{equation}\label{error-(k+1)}
\be_{k}(\bW, \bbb; \bx):=\be_{k-1}^*(\bx)-(\bW\cN_{k-1}(\bx)+\bbb), \ \ \bx\in\bR^s,
\end{equation}
and find 
\begin{equation}\label{step(k+1)}
(\bW_{k}^*,\bbb_{k}^*):={\rm argmin}\{\|\be_{k}(\bW, \bbb; \cdot)\|^2: \bW\in  \bR^{t\times t}, \bbb\in \bR^t\}.
\end{equation}
Again, the error function \eqref{error-(k+1)} of grade $k$ does not involve an activation function for this layer.
Since the weight matrices $\bW_j^*$ and bias vectors $\bbb_j^*$, for $j=1,2,\dots, k-1$, involved in the neural network $\cN_{k-1}$ have been determined, \eqref{step(k+1)} is again a quadratic optimization problem with respect to $\bW$ and $\bbb$, which can be efficiently solved by existing algorithms.
With $\bW_{k}^*$ and $\bbb_{k}^*$ found,  we obtain that
\begin{equation}\label{ffk}
    \ff_{k}(\bx):=\bW^*_{k}\cN_{k-1}(\bx)+\bbb_{k}^*, \ \ \bx\in\bR^s,
\end{equation}
which approximates $\be_{k-1}^*$ and it is a part of the residual  information leftover from learning of all the previous grades. Once again, $\ff_{k}$ is an ``affine map'' (or linear function) of $\cN_{k-1}$. However, $\ff_{k}$ is not a linear function of $\bx$ since $\cN_{k-1}$ involves the activation function $\sigma$.  We then define the optimal error of grade $k$ by
\begin{equation}\label{Optimal-error-k+1}
    \be_{k}^*(\bx):= \be_{k}(\bW_{k}^*, \bbb_{k}^*; \bx), \ \ \bx\in\bR^s
\end{equation}
and the neural network of grade $k$ by
\begin{equation}\label{network-(k+1)}
    \cN_{k}(\bx):=\sigma(\ff_{k}), \ \ \bx\in\bR^s.
\end{equation}
When learning of grade $l$ is completed, the deep neural network learned is given by
\begin{equation}\label{total-ff_k+1}
    \overline{\ff}_{l}:=\sum_{k=1}^{l}\ff_k.
\end{equation}
Unlike the standard neural network, which has only one neural network, the neural network $\overline{\ff}_{l}$ learned by the SAL model is the superposition of the neural networks learned in grade 1 through grade $l$. It also differs from the multi-grade deep learning model introduced in \cite{Xu2023} with every grade consisting of exactly one layer, where each grade solves a non-convex optimization problem since its objective function involves the activation function. The $l$ neural networks $\ff_k$, $k\in \bN_l$, are adaptive orthogonal basis functions for approximation of $\ff$.

\begin{figure}[H]
\centering
\includegraphics[width=0.5\textwidth, height=0.35\textwidth]{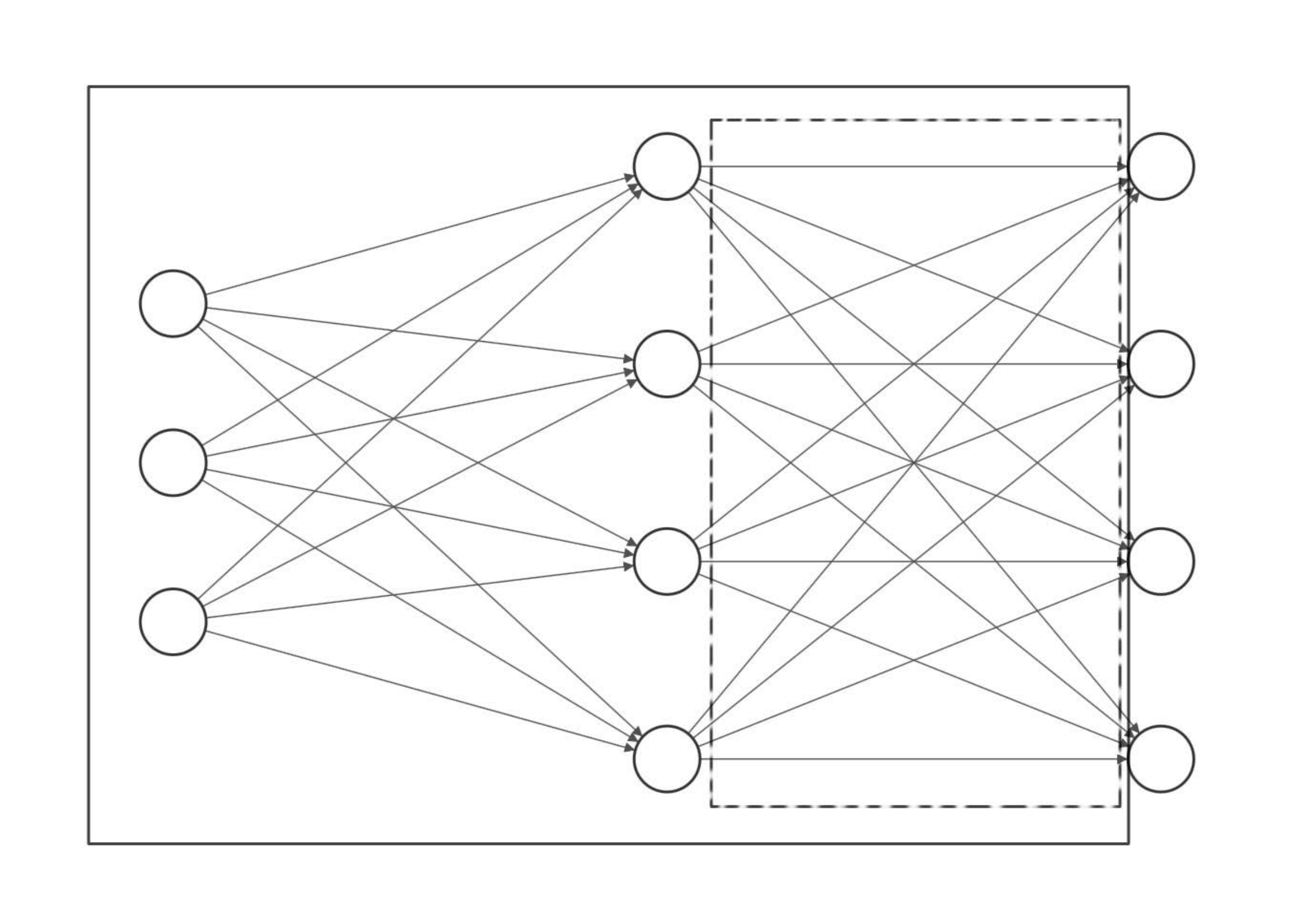}
\caption{Illustration of the Successive Affine Learning model}\label{SAL}
\end{figure}

The essential point of the SAL model is that the weight matrix and bias vector for each grade  are determined by an quadratic optimization problem which does not involve the activation function. 
We illustrate the SAL model in Figure \ref{SAL}. In the figure, the region embraced by the broken lines includes the parameters trained in the current grade and that bounded by the solid lines includes all layers contributed to the training of the current grade. The neurons on the most-right side are not involved in the training of the current grade. However, they will be involved in the training of the next grade, with a role as a ``basis'' determined by the previous grades.

The superiority of the SAL model described above over the standard deep learning is clear.
The  standard neural network of $n$ layers is learned by a single-grade learning model, where $n$ weight matrices and $n$ bios vectors are trained all together, by solving a highly non-convex optimization problem with a vast number of parameters, which would often suffer from the vanishing gradient issue or getting stuck in poor solutions. While the neural network $\overline{\ff}_{n}$ learned by the SAL model is constructed by solving a series of quadratic optimization problems. Specifically,  $\overline{\ff}_{n}$ is the superposition of $n$ neural networks, each of which adds on the top of the previously learned network a new layer with the weight matrix and bios vector trained by solving a quadratic optimization problem. Advantages of this construction include that we pay only the computational cost for solving $n$ quadratic  optimization problems, while we gain the expressiveness power of the nonlinear function compositions of neural networks. Moreover, unlike the standard deep learning model which requires differentiating the activation functions when solving the associated optimization problem, the SAL does not need to differentiate the activation function because the activation function is not involved in the optimization problem for the layer. This makes the SAL model very effective in numerical computation.

The SAL model is also suitable for learning a function from $m$ pairs of discrete data points  $\cD_m:=\{(\bx_j, \by_j)\}_{j=1}^m$, where $\bx_j\in \bR^s$ and $\by_j\in \bR^t$. We now modify the aforementioned model to fit this setting. In this case, the error function for grade 1 is now defined by
\begin{equation*} \label{error1-GD}
\be_1(\bW, \bbb; j):=\by_j-(\bW\bx_j+\bbb), \ \  j\in \bN_m,
\end{equation*}
and its discrete norm has the form
\begin{equation}\label{Error-Discrete1}
    \|\be_1(\bW, \bbb; \cdot)\|_m^2:=\sum_{j=1}^m\|\by_j-(\bW\bx_j+\bbb)\|_{\ell_2}^2.
\end{equation}
We then find
\begin{equation}\label{min1D}
(\bW_1^*,\bbb_1^*):=    {\rm argmin}\{\|\be_1(\bW,\bbb; \cdot)\|_m^2: \bW\in  \bR^{t\times s}, \bbb\in \bR^t\},
\end{equation}
which defines the affine function $\ff_1(\bx)$ and the neural network $\cN_1(\bx)$, $\bx\in\bR^s$, of grade 1 by   \eqref{ff1} and
\eqref{initial-network}, respectively, 
and the associated optimal error by 
$$
\be^*_1(j):=\be_1(\bW^*, \bbb^*; j), \ \  j\in\bN_m.
$$
For each grade $k=2,3,\dots,l$, we define the error function by
$$
\be_k(\bW, \bbb; j):=\be_{k-1}^*(j)-(\bW\cN_{k-1}(\bx_j)+\bbb),
\ \  j\in \bN_m,
$$
and its discrete norm has the form
\begin{equation}\label{Error-Discretei}
\|\be_k(\bW, \bbb; \cdot)\|_m^2:=\sum_{j=1}^m\|\be_{k-1}^*(j)-(\bW\cN_{k-1}(\bx_j)+\bbb)\|_{\ell_2}^2.
\end{equation}
We find
\begin{equation}\label{step(k+1)-d}
(\bW_{k}^*,\bbb_{k}^*):={\rm argmin}\{\|\be_{k}(\bW, \bbb; \cdot)\|_m^2: \bW\in  \bR^{t\times t}, \bbb\in \bR^t\}.
\end{equation}
With $\bW_k^*$ and $\bbb_k^*$ found, we obtain the affine function $\ff_k(\bx)$ and the neural network $\cN_k(\bx)$, $\bx\in\bR^s$, of grade $k$, by \eqref{ffk} and \eqref{network-(k+1)}, respectively.
The associated optimal error of grade $k$ by
$$
\be^*_k(j):=\be_k(\bW^*, \bbb^*; j), \ \ j\in\bN_m.
$$
When $l$ grades of learning are completed, the outcome is the DNN $\overline{\ff}_l$ having the form \eqref{total-ff_k+1} learned from the data points $\cD_m$.

We may define the error function in terms of other norms such as the $L_1$ (or $\ell_1$), $L_p$ (or $\ell_p$) norms, the K-L divergence and the entropy, depending on specific applications. In the cases when the norm used for the error function is not the $L_2$ (or $\ell_2$) norm, instead of solving a quadratic optimization problem, we will solve a convex optimization problem for each grade to learn the affine map for the grade. The form of the convex optimization problem is determined by the type of the norm used in the definition of the error function.


To close this section, we propose a ``$1+l$'' hybrid multi-grade model for learning an approximation of function $\ff$. The proposed ``$1+l$'' hybrid model combines the multi-grade model described in \cite{Xu2023} with the SAL model introduced in this paper. Namely, the proposed model consists of a shallow neural network for grade 1 and the SAL model of $l$ grades for grades 2 to $l+1$. Specifically,  
for grade 1, we learn a shallow neural network 
$\cN_{k_1}$ of $k_1$ layers
by solving non-convex optimization problem \eqref{min1-G} and let $\ff_1=\cN_1:=\cN_{k_1}$.  
For $k>1$,  we successively solve quadratic/convex optimization problem \eqref{step(k+1)} and construct the affine function $\ff_k$ and its associated neural network $\cN_k$ as in equations \eqref{ffk} and \eqref{network-(k+1)}, respectively. We repeat the process $l$ times.  In this way, we solve {\it only} one non-convex optimization problem \eqref{min1-G}, where $k_1$ is a small positive integer, for a shallow neural network of $k_1$ layers,  and solve $l$ quadratic/convex optimization problems  \eqref{step(k+1)} for  updates. In learning of grade 1, we learn lower-level features (for example, in image processing,  edges) from the input data by solving a non-convex optimization problem and in learning of higher grades, we learn higher-level features (details)  by successively solving quadratic/convex optimization problems. The hybrid model may increase the approximation accuracy of the SAL model.



\section{Successive Affine Learning with the Average Pooling}\label{Average-Pooling}

The SAL model described in the last section requires that at each grade, the dimension of the range space of matrix $\bW_k^*$ must be equal to the dimension of the vector-valued function to be learned, since in each grade the error function must have the same dimension as the original function to be approximated. Hence, except for $k=1$, $\bW_k^*$ are all $t\times t$ square matrices.  For a given layer, allowing the row size of the weight matrix to be greater than its column size so that the number of neurons to be used in the layer can be greater than $t$
can enhance the expressiveness of the resulting neural network. Hence, we must address the issue that the row size of the weight matrix is greater than $t$. We require that this addition will not ruin the quadratic or convex nature of the resulting optimization problem for each grade.

Recall that pooling layers are often used in deep learning to down sample feature maps by summarizing the presence of features in patches of the feature map. Two commonly used pooling methods are the average pooling and the max pooling. The average pooling summarizes the average presence of a feature and the max pooling summarizes the most activated presence of a feature.
We propose to employ the average pooling operator to pull back the matrix size to $t$ so that we can compute the error function. An advantage of using the average pooling operator lies on the fact that such a choice will not ruin the quadratic or convex nature of the resulting optimization problem for training the weight matrix and the bias vector for the layer. We next describe the SAL model assisted by the average pooling operator.

We first recall the average pooling operator.
For an integer $\mu\ge 0$, the average pooling $\cP_\mu$ is the linear operator from $\bR^{d+\mu}$ to $\bR^d$, for any $d\in\bN$, defined by
\begin{equation}\label{averagepooling}
	(\cP_\mu\bx)_i:=\frac{1}{\mu+1}\sum_{j=0}^\mu x_{i+j},\ \ i\in \bN_d,\ \ \bx\in\bR^{d+\mu}.
\end{equation}
It can be seen that the matrix representation of the average pooling operator $\cP_\mu$ is of full row rank. Hence, $\cP_\mu$ maps from $\bR^{d+\mu}$ onto $\bR^d$.
In the signal processing community, the average pooling is also called the down sampling operator. In particular, when $\mu=0$, $\cP_0$ reduces to the identity operator $\cI:\bR^d\to\bR^d$.

We now describe the SAL with the average pooling.
Suppose that a sequence of matrix widths $m_n\in\bN$, $n\in\bN$, is chosen with $m_n\geq t$ and $m_0:=s$, for a neural network to be learned. The parameter $\mu$ in the pooling operator $\cP_\mu$ is determined by the matrix widths $m_n$ at the $n$-th grade.
For matrix $\bW_1\in \bR^{m_1\times m_0}$ and vector $\bbb_1\in \bR^{m_1}$, we define the  error function $\be_{1}^P(\bW_1,\bbb_1; \cdot):\bR^s\to\bR^t$ of grade 1 by
\begin{equation}\label{error-function1-pooling}
    \be_{1}^P(\bW_1,\bbb_1; \bx):=\ff(\bx)-\cP_{\mu_1}(\bW_1\bx+\bbb_1), \ \ \bx\in\bR^s,
\end{equation}
where $\mu_1:=m_1-t$. Here, in general, $m_1>t$, and when $m_1=t$, we have that $\mu=0$, that is, $\cP_\mu$ reduces to the identity operator. Note that the pooling operator $\cP_{\mu_1}$ involved in \eqref{error-function1-pooling}
reduces the size of the affine map from $m_1$ to $t$ so that the right-hand-side of equation \eqref{error-function1-pooling} is well-defined. We then find 
\begin{equation}\label{min1-pooling}
(\bW_1^*,\bbb_1^*):=    {\rm argmin}\{\|\be^P_{1}(\bW_1,\bbb_1; \cdot)\|^2: \bW_1\in  \bR^{m_1\times s}, \bbb_1\in \bR^{m_1}\}.
\end{equation}
Since the pooling operator $\cP_{\mu_1}$ is a specified linear operator,  \eqref{min1-pooling} is a quadratic minimization problem with respect to $\bW_1$ and $\bbb_1$. As in section 3, the quadratic optimization problem 
\eqref{min1-pooling} can be solved efficiently by existing algorithms. In other words,  adding a pooling layer to the affine map to be learned does not increase significantly the computational complexity in solving optimization problem \eqref{min1-pooling}, comparing to solving optimization problem \eqref{min1}.
With $\bW_1^*\in\bR^{m_1\times t}$ and $\bbb_1^*\in\bR^{m_1}$ obtained, we get the affine function 
\begin{equation}\label{ff1-A}
    \ff_1^P(\bx):=\cP_\mu(\bW^*_1\bx+\bbb_1^*), \ \ \bx\in\bR^s,
\end{equation}
which approximates $\ff$. The associated optimal error of grade 1 is then defined by
\begin{equation}\label{Opt-Error-P}
    \be_{1}^{P*}(\bx):=\be_1^P(\bW_1^*,\bbb_1^*; \bx), \ \ \bx\in\bR^s
\end{equation}
and the associated neural network of grade 1 is defined by
\begin{equation}\label{initial-network-average-Modify}
    \cN_1^P(\bx):=\sigma(\bW_1^*\bx+\bbb^*_1), \ \ \bx\in\bR^s.
\end{equation}
Clearly, unlike the neural network $\cN_1$ constructed in section 3, the neural network $\cN_1^P: \bR^s\to\bR^{m_1}$ is no longer equal to $\sigma(\ff_1)$. Note that the dimension of the vector-valued function $\cN_1^P$ is larger than that of $\ff_1^P$, expected to have better expressiveness. 


Suppose that for $k\geq 1$, the neural networks $\ff_k^P:\bR^s\to\bR^t$ and $\cN^P_k: \bR^s\to\bR^{m_k}$ of grade $k$ 
have been learned, with the optimal error $\be_{k}^{P*}: \bR^s\to\bR^t$ determined by the weight matrix $\bW_k^*$ and the bias vector $\bbb_k^*$, and we proceed to learn the weight matrix and bias vector of grade $k+1$. We choose $\mu_{k+1}:=m_{k+1}-t$. Then, the average pooling operator  $\cP_{\mu_{k+1}}$ for grade $k+1$ will map $\bR^{m_{k+1}}$ to $\bR^t$. For matrix $\bW_{k+1}\in \bR^{m_{k+1}\times m_k}$ and vector $\bbb_{k+1}\in \bR^{m_{k+1}}$, we
define the error function $\be_{k+1}^P(\bW_{k+1}, \bbb_{k+1}; \cdot): \bR^s\to \bR^t$ with the average pooling
 of grade $k+1$ by
\begin{equation}\label{error-(k+1)-A}
\be_{k+1}^P(\bW_{k+1}, \bbb_{k+1}; \bx):=\be_{k}^{P*}(\bx)-\cP_{\mu_{k+1}}(\bW_{k+1}\cN^P_k(\bx)+\bbb_{k+1}), \ \ \bx\in\bR^s.
\end{equation}
We then find 
\begin{equation}\label{step(k+1)-A}
(\bW_{k+1}^*,\bbb_{k+1}^*):={\rm argmin}\{\|\be_{k+1}^P(\bW_{k+1}, \bbb_{k+1}; \cdot)\|^2: \bW_{k+1}\in  \bR^{m_{k+1}\times m_k}, \bbb_{k+1}\in \bR^{m_{k+1}}\}.
\end{equation}
Once again, the average pooling operator $\cP_{\mu_{k+1}}$ will not ruin the quadratic nature of the optimization problem.
With the weight matrix $\bW_{k+1}^*\in\bR^{m_{k+1}\times m_k}$ and the bias vector $\bbb_{k+1}^*\in\bR^{m_{k+1}}$ found,
we obtain that
\begin{equation}\label{ff_k+1-A}
\ff_{k+1}^P(\bx):=\cP_{\mu_{k+1}}(\bW_{k+1}^*\cN^P_k(\bx)+\bbb^*_{k+1}), \ \ \bx\in\bR^s,
\end{equation}
which approximates the optimal error $\be_k^{P*}$ of grade $k$. Note that neural network $\ff_{k+1}^P$ is a vector-valued function mapping $\bR^s$ to $\bR^t$. We then define another neural network with the average pooling of grade $k+1$ by
\begin{equation}\label{network(k+1)-A-Modify}
\cN_{k+1}^P(\bx):=\sigma(\bW_{k+1}^*\cN_k^P(\bx)+\bbb^*_{k+1}), \ \ \bx\in\bR^s,
\end{equation}
which maps $\bR^s$ to $\bR^{m_{k+1}}$.
The optimal error function of grade $k+1$ is clearly given by
\begin{equation}\label{optimal-error-k+1-A}
    \be_{k+1}^{P*}(\bx):= \be_{k+1}^P(\bW_{k+1}^*, \bbb_{k+1}^*; \bx), \ \ \bx\in\bR^s.
\end{equation}
When the SAL with the average pooling of grade $k+1$ is completed, we have learned the neural network with the average pooling 
\begin{equation}\label{total-ff_k+1-A}
    \overline{\ff}_{k+1}^P:=\sum_{i=1}^{k+1}\ff_i^P,
\end{equation}
which approximates $\ff$ with the error
$$
\be_{k+1}^{P*}(\bx):=\ff(\bx)-\overline{\ff}_{k+1}^P(\bx), \ \ \bx\in \bR^s.
$$
Therefore, the SAL model leads to the orthogonal expansion of $\ff$:
\begin{equation}\label{Orthogonal-expansion}
 \ff=  \sum_{i=1}^{k+1}\ff_i^P+\be_{k+1}^{P*},
\end{equation}
where $\ff_i$, $i=1,2,\dots,K+1$ and $\be_{k+1}^{P*}$ are mutually orthogonal.

Using the pooing operator in the SAL model is crucial to increase the accuracy of the resulting neural network. 
The SAL with the average pooing allows us to expand the sizes of the weight matrices and the bias vectors of the resulting neural network, and thus to enhance the expressiveness of the learned function.
Note that when $\mu=0$, the SAL with the trivial pooling operator $\cP_0$ reduces to the SAL without pooling.

A comment on the pooling operator is in order.
One may replace the average pooing used in SAL by other types of pooling, for example, the max pooling. When the max pooling is used, the resulting models to learn the weight matrices and bias vectors are no longer quadratic optimization problems. One may also substitute the average pooling by a matrix of an appropriate sizes. For instance, in learning of grade $k+1$, one may replace $\cP_{\mu_{k+1}}$ by a matrix $\bP\in\bR^{t\times m_{k+1}}$, with partial or all entries fixed. When the matrix contains free parameters, again the resulting model is not a quadratic optimization. Preliminary numerical results presented in this paper show that the average pooling works well. It is our future research project to investigate possible uses of other types of pooling operators.


\section{Analysis of the Successive Linear Learning Model}

In this section, we provide rigorous theoretical analysis for the proposed SAL model. We will show that the function representation generated by the SAL model enjoys the Pythagorean identity and the Parseval identity. These results make the ``harmonic analysis'' of DNNs possible. Moreover, we prove that the SAL
model without pooling either terminates after a finite number of grades or the optimal error functions of the grades are strictly decreasing in their norms.

We first represent a given function $\ff\in L_2(\DD,\bR^t)$ in terms of the superposition of the neural networks learned by the SAL model with the average pooling operator and the optimal error function.

\begin{theorem}
If $\ff\in L_2(\DD,\bR^t)$, then $\ff$ has the representation, for each $k\in\bN$,
\begin{equation}\label{representation_of_ff}
\ff(\bx)=\sum_{j=1}^k \cP_{\mu_j}\left[\bW_j^*\left(\bigodot_{i=1}^{j-1}\sigma(\bW_i^* \cdot+\bbb_i^*)\right)(\bx)+\bbb_j^*\right]+\be^{P*}_k(\bx), \ \ \bx\in\DD,
\end{equation}
where $\cP_{\mu_j}$ is the average pooing operator.
\end{theorem}
\begin{proof}
By equation \eqref{total-ff_k+1-A} with $\be_{k}^{P*}:=\ff-\overline{\ff}_{k}^P$, we have that 
\begin{equation}\label{ff-expansion}
\ff=\sum_{j=1}^k\ff_j^P+\be_{k}^{P*}.
\end{equation}
It suffices to show for all $j\in\bN_k$ that
\begin{equation}\label{ff^P_j}
\ff_j^P(\bx)=\cP_{\mu_j}\left[\bW_j^*\left(\bigodot_{i=1}^{j-1} \sigma(\bW_i^* \cdot+\bbb_i^*)\right)(\bx)+\bbb_j^*\right].
\end{equation}
We will establish formula \eqref{ff^P_j} by showing
\begin{equation}\label{cN^P_j}
\cN_j^P(\bx)=\bigodot_{i=1}^{j} \sigma(\bW_i^* \cdot+\bbb_i^*)(\bx),
\end{equation}
since formula \eqref{ff^P_j} follows directly from \eqref{cN^P_j} and \eqref{ff_k+1-A}.
We now prove formula \eqref{cN^P_j} by induction on $j$. For $j=1$, by definition \eqref{initial-network-average-Modify}, we clearly have formula \eqref{cN^P_j} with $j=1$. We assume that formula \eqref{cN^P_j} holds true for $j$ and proceed to the case $j+1$. 
By substituting formula \eqref{cN^P_j} into the right-hand-side of equation \eqref{network(k+1)-A-Modify} with $k:=j$, we establish formula \eqref{cN^P_j} with $j$ being replaced by $j+1$.
%
\end{proof}

We next study the sequence of optimal error functions $\be_k^{P*}$, $k\in\bN$.
Note that in learning of grade $k$, we solve the quadratic 
optimization problem
\begin{equation}\label{step(k+1)-G}
(\bW_{k}^*,\bbb_{k}^*):={\rm argmin}\{\|\be_{k-1}^{P*}(\cdot)-\cP_{\mu_k}(\bW\cN_{k-1}^P +\bbb)\|^2: \bW\in  \bR^{m_k\times m_{k-1}}, \bbb\in \bR^{m_k}\}.
\end{equation}
We will rewrite problem \eqref{step(k+1)-G} as an orthogonal projection to a subspace. To this end,
for each $k\in\bN$, 
we let
\begin{equation}\label{SpaceA-k}
    \cA_k^P:={\rm span}\{\cP_{\mu_k}[\bW\cN_{k-1}^P(\cdot) +\bbb]:  \bW\in  \bR^{m_k\times m_{k-1}}, \bbb\in \bR^{m_k}\},
\end{equation}
with $\cN_0^P(\bx):=\bx$. 

\begin{lemma}\label{cN_k-inL2}
Let $\ff\in L_2(\DD,\bR^t)$ and $\cN_k^P$ be generated from $\ff$ by the SAL model with the average pooing operator. If $\sigma: \bR\to\bR$ 
is bounded on any bounded set $\DD\subset\bR$,
then 

(i) $\cN_k^P\in L_2(\DD, \bR^{m_k})$;

(ii) $\cA_k^P$ is a linear subspace of $L_2(\DD,\bR^t)$.
\end{lemma} 
\begin{proof}
(i) According to the hypothesis, there exists a positive constant $L$ such that 
$$
\|\sigma(\by)\|_{\ell_2}\leq L\ \ \mbox{for all}\ \ \by\in\tilde\DD,
$$
where $\tilde\DD$ is a bounded set mapped from $\DD$.
By the definition of the norm of the space $L_2(\DD,\bR^{m_k})$ and that of $\cN_k^P$, we have that
\begin{align*}
   \int_\DD \|\cN_k^P(\bx)\|_{\ell_2}^2d\bx
   &=\int_\DD\|\sigma(\bW^*_k\cN_{k-1}^P(\bx)+\bbb_k^*)\|_{\ell_2}^2d\bx \leq L {\rm meas}(\tilde\DD)<+\infty.
\end{align*}
That is, $\cN_k^P\in L_2(\DD,\bR^{m_k})$.

Item (ii) follows directly from Item (i).
\end{proof}

We find it helpful to re-express $\ff_k^P$ determined by minimization problem \eqref{step(k+1)-G} as an orthogonal projection.
%
 With the notation $\cA_k^P$, the minimization problem \eqref{step(k+1)-G} may be rewritten as
\begin{equation}\label{step(k+1)-G2}
\ff_k^P={\rm argmin}\{\|\be_{k-1}^{P*}-\bg\|^2: \bg\in\cA_k^P\}.
\end{equation}
That is, 
$\ff_{k}^P$ is the orthogonal projection of $\be^{P*}_{k-1}$ onto the subspace $\cA_k^P$.
We are now ready to present our first main result of this section.

\begin{theorem}\label{Theorem-error-Pooling} Let $\ff\in L_2(\DD,\bR^t)$, $\ff_k^P$,  $\cN_k^P$, $k\in\bN$, be generated by the SAL model with the average pooling operator, and $\be_k^{P*}$ $k\in\bN$, be the corresponding optimal error functions. The following statements hold true:

(i) For all $k\in\bN$,
\begin{equation}\label{Pythagorean}
    \|\be_{k}^{P*}\|^2=\|\be_{k+1}^{P*}\|^2+\|\ff_{k+1}^P\|^2,
\end{equation}
and $\ff_{k+1}^P=\b0$ for some $k\in\bN$ if and only if $\|\be_{k+1}^{P*}\|=\|\be_k^{P*}\|$.

(ii) For all $k\in\bN$,
\begin{equation}\label{error-nonincreasing}
    \|\be^{P*}_{k+1}\|\leq \|\be^{P*}_{k}\|,
\end{equation}
and the sequence $\|\be_k^{P*}\|$, $k\in\bN$, has a nonnegative limit.


(iii) For each $k\in\bN$, either $\ff_{k+1}^P=\b0$  or
\begin{equation}\label{error-decreasing}
    \|\be^{P*}_{k+1}\|< \|\be^{P*}_k\|.
\end{equation}

(iv) If $\cN_{k}^P= \b0$ for some $k\in\bN$, then $\|\be_{k+1}^{P*}\|=\|\be_{k}^{P*}\|$.

\end{theorem}
\begin{proof}
(i) By definitions \eqref{error-(k+1)-A},  \eqref{ff_k+1-A} and \eqref{optimal-error-k+1-A}, we observe that
$$
\be_k^{P*}(\bx)=\be_{k+1}^{P*}(\bx)+\ff_{k+1}^P(\bx), \ \ \bx\in\DD.
$$
By Item (ii) of Lemma \ref{cN_k-inL2}, $\cA_k^P$ is a linear subspace of $L_2(\DD,\bR^t)$.
Moreover, by the discussion prior to the statement of this theorem,  $\ff_{k+1}^P$ is the orthogonal projection of $\be^{P*}_k$ onto the subspace $\cA_{k+1}^P$. Thus, we have that
$$
\left<\be_{k+1}^{P*},\ff_{k+1}^P\right>=\left<\be_{k}^{P*}-\ff_{k+1}^P, \ff_{k+1}^P\right>=0.
$$
The last equality of the above equation holds because $\ff_{k+1}^P\in \cA_k^P$ and $\cA_k^P$ is a linear subspace of the Hilbert space $L_2(\DD,\bR^t)$. The equality follows from the characterization (see, for example \cite{Deutsch, Powell}) of the orthogonal projection onto a linear subspace of a Hilbert space.
The Pythagorean theorem of the orthogonal projection implies that equation \eqref{Pythagorean} holds true.

Part 2 of Item (i) follows directly from \eqref{Pythagorean}.

(ii) Inequality \eqref{error-nonincreasing} is a direct consequence of  \eqref{Pythagorean}. By inequality \eqref{error-nonincreasing}, the sequence $\|\be_k^{P*}\|$, $k\in\bN$, is nonincreasing and bounded below by zero. Therefore, it has a nonnegative limit.

(iii) It suffices to prove that if $\ff_{k+1}^P\neq \b0$ for some $k\in\bN$, then inequality \eqref{error-decreasing} must hold. 
Since $\ff_{k+1}^P\neq \b0$ for the index $k$, we obtain that $\|\ff_{k+1}^P\|>0$, and thus from equation \eqref{Pythagorean}, we conclude that  \eqref{error-decreasing} must hold.

(iv)
If $\cN_{k}^P=0$ for some $k\in\bN$, then by definition \eqref{error-(k+1)-A} we observe that
$$
\be_{k+1}^P(\bW,\bbb;\bx)=\be^{P*}_{k}(\bx)-\cP_{\mu_{k+1}}\bbb, \ \ \mbox{for}\ \ \bW\in\bR^{m_{k+1}\times m_k}, \ \bbb\in\bR^{m_{k+1}}.
$$
In this case, the solution of the minimization problem \eqref{step(k+1)-A} is given by $(\bW_{k+1}^*, \bbb_{k+1}^*)$, where $\bW_{k+1}^*$ is any element in $\bR^{m_{k+1}\times m_k}$. Once again, since $\cN_{k}^P=0$ and
$$
\be^{P*}_{k}(\cdot)=\be^{P*}_{k-1}(\cdot)-\cP_{\mu_k}[\bW_{k}^*\cN_{k-1}(\cdot)+\bbb_{k}^*],
$$
we obtain that
\begin{align*}
\|\be_{k+1}^{P*}\|
    &=\|\be^{P*}_{k}(\cdot)-\cP_{\mu_{k+1}}\bbb^*_{k+1}\|\\
    &=\|\be^{P*}_{k-1}(\cdot)-\cP_{\mu_k}[\bW_{k}^*\cN_{k-1}(\cdot)+\bbb_{k}^*]-\cP_{\mu_{k+1}}\bbb^*_{k+1}\|
\end{align*}
for the index $k$. Because the average pooling operator $\cP_\mu:\bR^{t+\mu}\to \bR^t$ has the matrix representation  of full row rank, for any vector $\bbb\in\bR^t$, there exists a vector $\bc\in\bR^{t+\mu}$ such that 
$
\bbb=\cP_{\mu}\bc.
$
This together with the fact $\cP_{\mu_{k+1}}\bbb^*_{k+1}\in\bR^t$ implies that there exists some vector $\tilde\bbb\in\bR^{m_k}$ such that
$$
\cP_{\mu_{k+1}}\bbb^*_{k+1}=\cP_{\mu_k}\tilde\bbb.
$$
Therefore, we have that
$$
\|\be_{k+1}^{P*}\|
    =\|\be^{P*}_{k-1}(\cdot)-\cP_{\mu_k}[\bW_{k}^*\cN_{k-1}(\cdot)+\bbb_{k}^*+\tilde\bbb]\|.
$$
By the construction of $\bW_k^*$ and $\bbb_k^*$, we observe that
\begin{align*}
\|\be_{k+1}^{P*}\|
    &\geq \|\be^{P*}_{k-1}(\cdot)-\cP_{\mu_k}[\bW_{k}^*\cN_{k-1}(\cdot)+\bbb_{k}^*]\|=\|\be^{P*}_{k}\|.
\end{align*}
That is, for this particular index $k$, we have that
\begin{equation}\label{Rev}
    \|\be_{k+1}^{P*}\|   \geq\|\be^{P*}_{k}\|.
\end{equation}

On the other hand,  by Item (ii) of this theorem, we have that 
$$
\|\be_{j+1}^{P*}\|\leq \|\be^{P*}_{j}\|,\ \  \mbox{for all}\ \ j\in\bN.
$$
In particular, for the index $k$, we have that $\|\be_{k+1}^{P*}\|\leq \|\be^{P*}_{k}\|$, which together with inequality \eqref{Rev} leads to the equation $\|\be_{k+1}^{P*}\|= \|\be^{P*}_{k}\|$, for this particular index $k$.
\end{proof}

Note that equation \eqref{Pythagorean} is the Pythagorean identity for the neural networks learned in grade $k+1$.

We next establish the Parseval identity for functions generated by the SAL model.

\begin{theorem}\label{Theorem-error-Pooling3}
Let $\ff\in L_2(\DD,\bR^t)$. If $\ff_k^P$, $k\in\bN$, is the sequence generated by the SAL model with the average pooling, and $\be_k^{P*}$, $k\in\bN$, is the corresponding sequence of optimal error functions,
then, for all $k\in\bN$,
\begin{equation}\label{k-term-Pythagorean}
\|\ff\|^2=\sum_{j=1}^k\|\ff_j^P\|^2+    \|\be_{k}^{P*}\|^2.
\end{equation}
Moreover, if $\|\be_{k}^{P*}\|\to 0$ as $k\to \infty$, then
\begin{equation}\label{infinite-term-Pythagorean}
\|\ff\|^2=\sum_{j=1}^\infty\|\ff_j^P\|^2
\end{equation}
and 
\begin{equation}\label{ff-infinite-expansion}
\ff=\sum_{j=1}^\infty\ff_j^P,
\end{equation}
where equation \eqref{ff-infinite-expansion} holds in the sense of the  $L_2$ convergence.
\end{theorem}
\begin{proof}
From the construction \eqref{ff1-A} of approximation $\ff_1^P$ of grade 1 and the definition \eqref{Opt-Error-P} of the associated optimal error function $\be_1^{P*}$, we observe that
\begin{equation}\label{Parseval1}
    \ff=\ff_1^P+\be_1^{P*}\ \  \mbox{and}\ \ \left<\be_1^{P*}, \ff_1^P\right>=0.
\end{equation}
It follows from \eqref{Parseval1} that 
$$
\|\ff\|^2=\|\ff_1^P\|^2+\|\be_1^{P*}\|^2.
$$
By employing the above equation and repeatedly using equation \eqref{Pythagorean} in Theorem \ref{Theorem-error-Pooling}, we obtain for all $k\in\bN$ that
\begin{align*}
\|\ff\|^2&=\|\ff_1^P\|^2+\|\be_1^{P*}\|^2\\
    &=\|\ff_1^P\|^2+\|\ff_2^P\|^2+\|\be_2^{P*}\|^2\\
    &=\sum_{j=1}^k\|\ff_j^P\|^2+\|\be_k^{P*}\|^2,
\end{align*}
which gives equation \eqref{k-term-Pythagorean}.

Equation \eqref{k-term-Pythagorean} implies that 
\begin{equation}\label{Parseval-Equality**}
    \sum_{j=1}^k\|\ff_j^P\|^2\leq \|\ff\|^2<+\infty,\ \ \mbox{for all}\ \ k\in\bN.
\end{equation}
Clearly, the sequence 
$$
F_k:= \sum_{j=1}^k\|\ff_j^P\|^2
$$
is nondecreasing and bounded above according to \eqref{Parseval-Equality**}. Hence, the sequence $F_k$, $k\in\bN$, has a limit as $k\to\infty$.
That is,
\begin{equation}\label{Parseval-Equality2}
    \sum_{j=1}^\infty\|\ff_j^P\|^2<+\infty.
\end{equation}
If $\|\be_{k}^{P*}\|\to 0$ as $k\to \infty$, letting  $k\to \infty$ in the both sides of equation  \eqref{k-term-Pythagorean} with considering \eqref{Parseval-Equality2} yields the Parseval identity \eqref{infinite-term-Pythagorean}.

Finally, by equation \eqref{ff-expansion} and by the hypothesis that $\|\be_{k}^{P*}\|\to 0$ as $k\to \infty$, we conclude that
$$
\left\|\ff-\sum_{j=1}^k\ff_j^P\right\|_2=\|\be_k^{P*}\|_2\to 0, \ \ \mbox{as} \ k\to\infty.
$$
This leads to series \eqref{ff-infinite-expansion}.
\end{proof}


For the hypothesis that $\|\be_{k}^{P*}\|\to 0$ as $k\to \infty$ in Theorem \ref{Theorem-error-Pooling} to satisfy, it requires additional information on the activation. We postpone investigating this issue to a future project.

To close this section, we consider the issue of when the SAL model will terminate in a finite number of grades. We provide an answer to this question in the following theorem for the special case when the pooling is the identity operator.

\begin{theorem}\label{Theorem-error-Pooling2} Let $\ff\in L_2(\DD,\bR^t)$.
If the activation function $\sigma$ satisfies $\sigma(0)=0$, then 
either the SAL model terminates after a finite number of grades or the norms of its optimal error functions strictly decrease to a limit as the grade number increases.
\end{theorem}
\begin{proof}
We let $\ff_k$,  $\cN_k$, $k\in\bN$, be generated by the SAL model without pooling, and $\be_k^{*}$ $k\in\bN$, be the corresponding optimal error functions. 
We consider two different cases. For the first case, we suppose that $\ff_{k}=\b0$ for some $k\in\bN$. Since $\sigma(0)=0$, by the definition of $\cN_{k}$ and the assumption that  $\ff_{k}=\b0$, we conclude that $\cN_{k}=\b0$ for the particular index $k$. According to Item (iv) Theorem \ref{Theorem-error-Pooling}, we find that 
$$
\|\be^{*}_{k+1}\|=\|\be^{*}_{k}\|, \ \ \mbox{for this particular}\ \ k. 
$$
This equation together with part two of Item (i) of Theorem \ref{Theorem-error-Pooling} ensures that $\ff_{k+1}=\b0$ for this particular index $k$. Repeating this process gives rise to the assertion that $\ff_n=\b0$, for all $n\geq k$. Therefore, the SAL model with the average pooling terminates after $k$.

We next consider the second case.
Suppose that the first case does not take place. Then, we must have that $\ff_{k}\neq\b0$ for all $k\in\bN$. In this case, it follows from Item (iii) of Theorem \ref{Theorem-error-Pooling}  that the strict inequality \eqref{error-decreasing} must hold for all $k\in\bN$. In other words, the norm sequence of optimal error functions is strictly decreasing. Moreover, by Item (ii) of Theorem  \ref{Theorem-error-Pooling}, the norm sequence of optimal error functions has a limit.
\end{proof}

Theorem \ref{Theorem-error-Pooling2} ensures convergence of the SAL process in the sense that either
it terminates after a finite number of grades or the norms of its optimal error functions strictly decrease to a limit as the grade number increases.
We note that
Theorem \ref{Theorem-error-Pooling2} requires the activation function $\sigma$ to satisfy the condition that $\sigma(0)=0$. Many activation functions such as ReLU,  leaky ReLU and Tanh satisfy this condition. When an activation function $\sigma(0)\neq 0$, we may define 
$$
\tilde\sigma(x)=\sigma(x)-\sigma(0), \ \ x\in\bR.
$$
Then, for the modified activation function $\tilde\sigma$, we have that $\tilde\sigma(0)=0$. This indicates that the condition in Theorem  \ref{Theorem-error-Pooling2} seems not a very restricted one.

\section{Smoothing of Learning Solutions}

We now turn to smoothing of the optimal error function or the learned solution of a grade in the SAL model.
The approximation accuracy of the SAL model may be constrained by noise contaminated in the optimal error function or the learned solution of a grade.  
Recall that starting grade 2, the SAL model learns the weight matrix and the bias vector of a grade from the optimal error function of the previous grade, which is defined by the subtraction of two functions. 
When the number of the grade is high, the error function which is the subtraction of two functions can be oscillatory. Direct learning from an oscillatory function may result in a low accuracy. To address this issue, we propose to apply a smoothing operator to the optimal error function or the learned solution of the current grade, before proceeding to learning of the next grade. 

A commonly used smoothing operator is an operator defined by the Gaussian function. We first describe the one dimensional case, which can be extended to a higher dimensional case without difficulties. The one dimensional Gaussian function has the form
$$
G(x):=\frac{1}{\sqrt{2\pi}}e^{-\frac{x^2}{2}}, \ \ x\in \bR.
$$
For $\tau>0$, we let
$$
G_\tau(x):=\frac{1}{\tau}G\left(\frac{x}{\tau}\right), \ \ x\in\bR.
$$
It is well-known that $G_\tau$ is an approximate identity, (see, for example, \cite{I.Daubechies, Frazier}). That is, if $f\in L_1(\bR)$, then for every Lebesgue point $x$ of $f$, there holds
\begin{equation}\label{Aprox-I}
    \lim_{\tau\to 0^+}(G_\tau *f)(x)=f(x),
\end{equation}
where $*$ denotes the convolution (Theorem 5.11 of \cite{Frazier}). It can be verified that for $\tau>0$, the function $G_\tau*f$ is sufficiently smooth and according to formula \eqref{Aprox-I}, when $\tau$ is small, $G_\tau*f$ is a good approximation of the function $f$. Thus, the convolution of $G_\tau$ provides us an ideal smoothing operator. We can construct a smooth operator for a multivariate function  via tensor product and we use $\cG_\tau$ to denote the resulting smoothing operator. 

We now describe the smoothing process.
Suppose that $\ff_i^P$ is a neural network learned in grade $i$. We apply the smoothing operator $\cG_\tau$ to either $\ff_i^P$ or $\be_i^{P*}$ before we proceed to learning of grade $i+1$. The smooth operator can alleviate the oscillation of the functions, leading to improvement of approximation accuracy. For example, when the smoothing operator is applied to the learned function $\ff_i^P$ of grade $i$, we obtain the smoothed approximation
\begin{equation}\label{smoothed-function}
    \ff_{i, \tau}^P:=\cG_\tau \ff_i^P.
\end{equation}
By the property of the Gaussian function $G_\tau$, we observe that the function $\ff_{i,\tau}^P$ is sufficiently smooth. With the smoothed learned function, we define a new optimal error function $\be_{i,\tau}^{P*}$ from which we learn a function $\ff_{i+1,\tau}^P$ for grade $i+1$. We then apply the smoothing operator to $\ff_{i+1,\tau}^P$ and we use the same notation for the resulting function by a bit of abuse of notation. For a different grade, we may choose a different smoothing parameter $\tau$. 

Theoretical results presented in section 5 for learned function $\ff_i^P$ may be extended for the smoothed learned function $\ff_{i,\tau}^P$, due to the approximation property \eqref{Aprox-I}. We leave detailed proofs of such results to interested readers.

The smoothed learned function $\ff_{i,\tau}^P$ is intimately related to a regularized solution in a  Hilbert space  determined by the Gaussian kernel. Since the Gaussian kernel is universal \cite{MXZ}, a regularized solution in the space determined by the kernel has a nice approximation property. 
Suppose that the neural network $\cN_{i-1}^P$ has been learned. For some $\tau>0$, we let
$$
G_{i, \tau}(\bx'):=\int_{\bR^s}G_\tau(\bx'-\bx)\cN_{i-1}^P(\bx)d\bx
$$
and define
$$
\HH_{i,\tau}:={\rm span} \{\cP_{\mu_i}(\bW_iG_{\tau,i}(\cdot)+\bbb_i): \bW_i\in\bR^{m_1\times m_{i-1}}, \bbb_i\in\bR^{m_i}\}.
$$
Given the error function $\be_{i-1}^{P*}$, one can learn $\ff_{\tau, i}^P$ by solving the regularization problem
\begin{equation}\label{Smoothed-learning}
    \min\{\|\be_{i-1}^{P*}-\ff_i\|_2^2+\lambda \|\ff_i\|^2_{\HH_{i, \tau}}: \ff_i\in \HH_{i,\tau}\}.
\end{equation}
Once again, the regularization problem \eqref{Smoothed-learning}
is a quadratic optimization problem. Instead of solving the quadratic optimization problem \eqref{step(k+1)-A}, we solve 
the regularization problem \eqref{Smoothed-learning}, which gives us a smoothed learned function. Notice that
the smoothed learned function $\ff_{i,\tau}^P$ defined by \eqref{smoothed-function} may be seen as a certain solution of the regularization problem \eqref{Smoothed-learning}. 
We postpone a systemic investigation of this connection to a future project.

\section{Computational Issues}

We discuss in this section several critical computational issues related to implementation of the SAL model. They include the ``optimal choice'' of activation function, fast smoothing of  the learned solution of a grade and efficient algorithms for solving the quadratic/convex optimization problems that appear in the SAL model. Properly addressing these issues contributes positively to the success of the SAL model for learning of a DNN.

An issue crucial for the effectiveness of the SAL model is the choice of activation functions for each grade. One may use a fixed predetermined activation function for all grades in the SAL model, and may also change to a different activation function in a certain grade. The SAL model may be more effective if we choose activation functions from a linear combination of a collection of activation functions according to given data for different grades of learning. Since in each grade, the SAL model solves a quadratic/convex optimization problem, it is convenient for us to choose an activation function by solving another quadratic optimization problem after the weight matrix and bias vector have been chosen for the current grade.

We propose to use an ``optimal combination'' of a predetermined collection of activation functions 
$\{\sigma_j: j=1, 2, \dots, L\}$ for the activation function of grade $k$. Specifically, in grade $k$ we suppose that the weight matrix $\bW^*_k$ and bias vector $\bbb_k^*$ have been learned
from $\be_{k-1}^{P*}$. At this step, $\ff_k^P$ and $\be_k^{P*}$ have been found. Instead of picking a fixed activation function to define $\cN_k^P$, we wish to choose an appropriate activation function from a linear combination of the activation functions $\sigma_1, \sigma_2, \dots, \sigma_L$ for this grade, with the coefficients $\alpha^*_1, \alpha_2^*, \dots,\alpha^*_L$ determined by the optimal error function $\be_{k}^{P*}$ of grade $k$.
Namely, we find the parameters $\alpha^*:=[\alpha^*_1, \alpha_2^*, \dots,\alpha^*_L]^\top\in\bR^L$ by solving the quadratic minimization problem
\begin{equation}\label{step(k)-sigma}
\min\left\{\left\|\be_{k}^{P*}(\cdot)-\cP_{\mu_k}\sum_{j=1}^L\alpha_j\sigma_j(\bW_k^*\cN_{k-1}^P(\cdot) +\bbb_k^*)\right\|^2: \alpha:=[\alpha_1,\dots,\alpha_L]^\top\in \bR^L\right\},
\end{equation}
and then we define 
$$
\sigma_{\alpha^*}:=\sum_{j=1}^L\alpha^*_j\sigma_j
$$
as the optimal activation function of grade $k$.
The neural network (with the average pooling) with the optimal activation function $\sigma_{\alpha^*}$ of grade $k$ is now defined by
\begin{equation}\label{network(k)-A-Modify}
\cN_{k,\alpha^*}^P(\bx):=\sigma_{\alpha^*}(\bW_{k}^*\cN_{k-1}^P(\bx)+\bbb^*_{k}), \ \ \bx\in\bR^s
\end{equation}
and the optimal error function of grade $k$ is updated by
\begin{equation}\label{optimal-error-k-A}
    \be_{k+1,\alpha^*}^{P}(\bW_{k+1},\bbb_{k+1};\bx):= \be_{k}^{P*}(\bx)-\cP_{\mu_{k+1}}(\bW_{k+1}\cN_{k,\alpha^*}^P(\bx)+\bbb_{k+1}), \ \ \bx\in\bR^s.
\end{equation}
The weight matrix $\bW_{k+1}^*$ and the bias vector $\bbb_{k+1}^*$ of grade $(k+1)$ will be found by solving the optimization problem \eqref{step(k+1)-A} with the objective function $\be_{k+1}^{P}(\bW_{k+1},\bbb_{k+1}; \cdot)$ being replaced by  $\be_{k+1,\alpha^*}^{P}(\bW_{k+1},\bbb_{k+1}; \cdot)$.

We now discuss computing the smoothed learned function defined by equation \eqref{smoothed-function}. In numerical computation, computing $\ff_{i,\tau}^P$ requires numerical integration. After using a numerical quadrature scheme, the right-hand-side of \eqref{smoothed-function} becomes a discrete convolution. When the quadrature nodes are chosen to be equal-spaced, one can apply the fast Fourier transform (FFT) to the resulting discrete convolution and compute it by utilizing a fast algorithm. To apply FFT, one may need to make appropriate boundary extension of the discrete form of $\ff_i^P$.

Finally, we turn to addressing solving the optimization problems for the SAL model.
All optimization problems involved in the SAL model, including those for determining the weight matrices and bias vectors, and those for choosing the activation functions, are quadratic/convex minimization problems. They are typical convex optimization problems with smooth gradients.
Hence, they can be efficiently solved by employing the Nesterov algorithm. When sparse regularization is needed, the corresponding sparse regularization problems of these optimization problems may be solved by using an FISTA type algorithm \cite{Beck}. Both the Nesterov and FISTA algorithms have a ${\cal O}(1/j^2)$ convergence rate, where $j$ is the number of iterations.
When implementing the Nesterov algorithm for solving the optimization problem \eqref{step(k+1)-G}, one needs to estimate the step-sizes of the iterations, which are related to the Lipschitz constant of the gradient of the objective function of the optimization problem. From the definition of the objective function of the optimization problem, it is clear that this can be done by computing the value of the neural network $\cN_{k-1}^P$ which has been obtained before solving the optimization problem \eqref{step(k+1)-G}. When the optimization problem is quadratic, one may recast it into a linear system, which can be efficiently solved by the conjugate gradient method or the
preconditioned conjugate gradient method.


\section{Numerical Examples}
In this section, we present proof-of-concept examples to test the numerical performance of the proposed SAL model in comparison with the standard single-grade (SSG) deep learning model. We consider approximating a non-differentiable  function and an oscillatory vector-valued function by deep neural networks. All the experiments reported in this section are performed with Python on the First Gen ODU HPC Cluster, where computing jobs are randomly placed on an X86\_64 server with the computer nodes Intel(R) Xeon(R) CPU E5-2660 0  @ 2.20GHz (16 slots), Intel(R) Xeon(R) CPU E5-2660 v2  @ 2.20GHz (20 slots), Intel(R) Xeon(R) CPU E5-2670 v2  @ 2.50GHz (20 slots), Intel(R) Xeon(R) CPU E5-2683 v4 @ 2.10GHz (32 slots).

In our experiments, for the SAL model, we solve the quadratic optimization problems of all grades by using the Nesterov algorithm, and for the SSG model, we solve the non-convex optimization problems by using the Adam algorithm \cite{Kingma} with learning rate $\alpha$ (to be specified later) and with initial guesses determined by the method proposed in \cite{He2015}. 


The training and testing data for the numerical examples for approximation of function $\ff$ are described as follows:

\noindent\textbf{Training data}: $\{(x_n, y_n)\}_{n = 1}^m \subset [a-\delta, b+\delta] \times \mathbb{R}^t$, where $x_n$'s are equally spaced on $[a-\delta, b+\delta]$, and for given $x_n$, the corresponding $y_n$ is computed by
$y_n := \ff(x_n)$. Here, $\delta\geq 0$ is chosen for possible boundary extension.

\noindent\textbf{Testing data}: $\{(x'_n, y'_n)\}_{n = 1}^{m'} \subset [a, b] \times \mathbb{R}^t$, where $x'_n$'s are random uniform distribution on $[a, b]$ and for given $x'_n$, the corresponding $y'_n$ is computed by $y'_n := \ff(x'_n)$. To avoid randomness, we use $numpy.random.seed(1)$

The relative squared error on the training data for prediction $\hat y_n$ of $y_n$ in $\mathbb{R}^t$ is defined by 
$$
\mathrm{rse(train)}:= \frac{\sum_{n = 1}^N\|\hat{y}_n - y_n\|^2}{\sum_{n = 1}^N\|y_n\|^2}
$$
Likewise, for an approximation $\hat{y}_n$ of $y'_n$, we define the relative squared error on the testing data by
\begin{equation*}
\mathrm{rse(test)}:= \frac{\sum_{n = 1}^{N'}\|\hat{y}_n - y'_n\|^2}{\sum_{n = 1}^{N'}\|y'_n\|^2}.
\end{equation*}

For numerical implementation of the SAL model, the smoothing process \eqref{smoothed-function} is conducted in a discrete form obtained from numerical integration of the smoothing operator. Specifically, a learned function (or a component) $f_j$ of grade $j$ is smoothed by the local discrete smoothing operator
\begin{equation}\label{D-Smoothing}
    \hat f_j(x) :=\frac{b_x-a_x}{M} \sum_{i=1}^{M}G_\tau(x-y_i)f_j (y_i), \ \ x\in [a_x, b_x],
\end{equation}
where $M$ denotes the number of nodes used for the numerical integration of the smoothing operator, $y_i := \frac{b_x-a_x}{M}i +a_x$, $i=1,2,\dots, M$, and the value of $\tau$ for different grades will be specified.
The values $a_x$ and $b_x$ that appear in equation \eqref{D-Smoothing} will be given for specific examples.

\subsection{Learning a non-differentiable function}

In this example, we learn the non-differentiable function
\begin{equation}\label{non-differentiable function}
 f(x) =\ff(x):= (x+1)\left(\phi_4 \circ \phi_3 \circ \phi_2 \circ \phi_1\right)(x),  \ \ x\in [-1, 1]   
\end{equation}
where
\begin{align*}
   & \phi_1(x) :=|\cos(\pi(x - 0.3)) - 0.7|,\ \
    \phi_2(x) := |\cos(2\pi (x - 0.5)) - 0.5|,\\
   & \phi_3(x) := -|x - 1.3| + 1.3,\ \
    \phi_4(x) := -|x - 0.9| + 0.9.
\end{align*}
For this example, $[a,b]:=[-1, 1]$, $\delta:=0.1$, $m:=5,001$, $m':=1,001$ and $t:=1$. Since all functions $\phi_j$ involve the absolute value function, $f$ is not differentiable.

In this example, we compare accuracy and training time of the SAL model with those of the SSG model. For the SAL model, we employ two network structures described below.


\noindent\textbf{SAL-1} composes of one input layer, 18 hidden layers of uniform width 300 and one output layer. 

\noindent\textbf{SAL-2} composes of one input layer, 28 hidden layers of  width 300 (layers 1-8), 500 (layers 9-12), 600 (layers 13-16), 700 (layers 17-20), 800 (layers 21-24), 900 (layers 25-28) and one output layer.

To ensure fair comparison, for the SSG model, we consider 21 different network 
structures, where 20 structures with uniform widths are listed in Table \ref{Structures-for-SSG} and structure SSG-21 with variable widths described below. Note that SSG-21 is similar to structure SAL-2 for the SAL model. 

\begin{table}[]
\centering
\caption{Network structures (hidden layers) for the SSG model.}
    \label{Structures-for-SSG}
    \scriptsize{
    \begin{tabular}{l||c|c}
    \hline
     
structure  & width
 & hidden layer \#   \\
  \hline
 SSG-1  & 50 & 6  \\
SSG-2 &  50 & 10  \\
SSG-3 &   50 & 14 \\
SSG-4 & 50 & 18   \\
SSG-5 & 50  & 20  \\
SSG-6 & 100  & 6   \\
SSG-7 & 100  & 10   \\
SSG-8 & 100  & 14   \\
SSG-9 & 100  & 18   \\
SSG-10 & 100  & 20   \\
SSG-11 & 200  & 6   \\
SSG-12 & 200  & 10   \\
SSG-13 & 200  & 14   \\
SSG-14 & 200  & 18   \\
SSG-15 & 200  & 20   \\
SSG-16 & 300  & 6   \\
SSG-17 & 300  & 10   \\
SSG-18 & 300  & 14   \\
SSG-19 & 300  & 18   \\
SSG-20 & 300  & 20   \\
\hline

\end{tabular}
    
    }
\end{table}

\noindent\textbf{SSG-21} composes of one input layer, 20 hidden layers of  width 300 (layers 2-8), 500 (layers 9-12), 600 (layers 13-16), 700 (layers 17-20),  and one output layer.

For both the SAL model and the SSG model,  we use $\frac{1}{2}\sin x+ \frac{1}{2}\cos x$ as the activation function for the first and second hidden layers, and use the ReLU activation function for the remaining hidden layers, and use the identity activation function for the output layer. The parameters involved in the discrete smoothing operator \eqref{D-Smoothing} for the SAL model are chosen as $a_x:=x-100h$, $b_x:=x+100h$, where $h:=\frac{2}{5000}$ and $M:=201$
for this example. For the SAL model, we only need to solve a quadratic optimization problem for each grade. The stopping
criterion for each grade is either the iteration number equal to 5,000 or the relative error between the
function values of two consecutive steps less than the given number $\epsilon$. The numbers of iterations reported in Tables \ref{table: example mul-dnn results} and \ref{table: example mul-dnn results:SAL-2} are the actual numbers used in the iterations.


\begin{table}[]
    \centering
    \caption{The SAL model with structure SAL-1 for learning function \eqref{non-differentiable function}.}
    \label{table: example mul-dnn results}
    \scriptsize{
    \begin{tabular}{c||c|c|c|c|c|c}
    \hline
       grade    &  $\tau$  &   $\epsilon$ &  iteration \#    & train time (second)& rse(train)  & rse(test) \\
    \hline  
       1      &0&     1e-6 & 24 & 0.34 & 1.50e-1 & 1.41e-1\\
       2   &0&     1e-6 & 1,418 & 15.65 &  1.51e-1 & 1.43e-1\\
       3   &0&    1e-6 & 257 & 2.58 &  1.52-1 & 1.44e-1\\
       4   & 6e-3 &    1e-7 & 4,999 & 56.83 & 5.71e-2 & 5.38e-2\\
       5  &6e-3& 1e-7 & 4,999 & 48.70 &3.27e-3 & 3.27e-3\\
    6   & 6e-3 & 1e-7 & 4,999 & 48.92 & 5.60e-4 & 5.30e-4\\
    7   &3e-3& 1e-7 & 4,999 &  48.58 &  1.12e-4 &  1.06e-4\\
    8   &3e-3& 1e-7 & 4,999 & 54.33 & 5.29e-5 &  5.27e-5\\
    9  &1e-3& 1e-7 & 4,999 &   48.71 & 2.75e-5 &  2.99e-5\\  
    10  &1e-3& 1e-7 & 3,677 &  41.97  & 2.21e-5 &  2.48e-5\\ 
    11  &4e-4& 1e-7 & 2,892 &  28.22 & 1.86e-5 &  2.10e-5\\ 
    12  &4e-4& 1e-7 & 1,748 &  18.22 & 1.58e-5 &  1.74e-5\\     
    13  & 4e-4 & 1e-7 & 1,138 & 11.80 & 1.33e-5 & 1.39e-5\\
    14  &4e-4& 1e-7 & 1,874 &  19.72 &  1.19e-5&  1.22e-5\\
    15   &2e-5 & 1e-7 & 1,437 & 14.49  & 1.09e-5 &  1.16e-5\\
    16   &2e-5 & 1e-7 & 861 &   9.13 & 9.88e-6 &  1.05e-5\\  
    17  &1e-5 & 1e-7 & 754 &  9.24 & 8.90e-6 &  9.69e-6\\ 
    18   &1e-5& 1e-7 & 1,175 &  13.81 & {\bf 8.19e-6} & {\bf  9.01e-6}\\ 
    \hline  
    total time  &   \multicolumn{6}{c}  {\bf 491.24}\\
    \hline
\end{tabular}
    
    }
\end{table}

\begin{table}[]
\centering
\caption{The SAL model with structure SAL-2 for learning function \eqref{non-differentiable function}.}
    \label{table: example mul-dnn results:SAL-2}

    \scriptsize{
    \begin{tabular}{c||c|c|c|c|c|c}
    \hline
       grade    &  $\tau$  &   $\epsilon$ &  iteration \#    & train time (second)& rse(train)  & rse(test) \\
    \hline   
       1      &0&      1e-6 & 24 & 0.33 & 1.50e-1 & 1.41e-1\\
       2  &0&     1e-6 & 1,418 &15.65 &  1.51e-1 & 1.43e-1\\
       3  &0&    1e-6 & 257 & 2.58 &  1.52e-1 & 1.44e-1\\
       4  &6e-3&    1e-7 & 4,999 &  56.83 & 5.71e-2 & 5.38e-2\\
       5  &6e-3& 1e-7 & 4,999 & 48.70 &3.27e-3 & 3.27e-3\\
    6  & 6e-3 & 1e-7 & 4,999 & 48.92 & 5.60e-4 & 5.30e-4\\
    7   &1e-3& 1e-7 & 4,999 & 49.37 &  8.74e-5 &  8.53e-5\\
    8   &1e-3& 1e-7 & 4,999 & 56.62 & 4.32e-5 &  4.31e-5\\
    9  &1e-3& 1e-7 & 4,397 &  87.54 & 3.24e-5 &  3.22e-5\\  
    10  &1e-3& 1e-7 & 4,999 &  113.28 & 1.51e-5 &  1.54e-5\\ 
    11  &1e-3& 1e-7 & 4,999 &   119.59 & 1.05e-5 &  1.08e-5\\ 
    12   &1e-3& 1e-7 & 4,999 &  119.06 & 8.57e-6 &  8.66e-6\\   
    13  &4e-4& 1e-7 & 4,999 &   155.41 & 7.03e-6 &  6.97e-6\\  
    14  &4e-4& 1e-7 & 4,732 &   151.11 & 5.14e-6 &  4.82e-6\\ 
    15  &4e-4& 1e-7 & 4,999 &   163.80 & 4.05e-6 &  3.70e-6\\
    16  &4e-4& 1e-7 & 4,137 &    135.94 & 3.25e-6 &  2.95e-6\\   
    17  &4e-4& 1e-7 & 2,259 &  89.00 & 2.75e-6 &  2.44e-6\\   
    18  &4e-4& 1e-7 & 4,621 &   189.52 & 2.21e-6 &  2.00e-6\\  
    19  &6e-5& 1e-7 & 2,993 &   115.34 & 1.81e-6 &  1.54e-6\\ 
    20  &6e-5& 1e-7 & 3,440 &   132.35 & 1.53e-6 &  1.24e-6\\ 
    21  &1e-5& 1e-7 & 2,131 &   82.05 & 1.31e-6 &  1.07e-6\\  
    22  &1e-5& 1e-7 & 3,287 &   141.12 & 1.11e-6 &  8.83e-7\\ 
    23  &1e-5& 1e-7 & 2,568 &   109.57 & 9.43e-7 &  7.63e-7\\
    24  &1e-5& 1e-7 & 2,552 &    108.21 & 8.33e-7 &  6.76e-7\\   
    25   &0& 1e-7 & 1,532 &  90.77 & 7.15e-7 &  6.00e-7\\   
    26  &0& 1e-7 & 2,416 &   154.81 & 6.18e-7 &  5.29e-7\\  
    27  &0& 1e-7 & 3,071 & 221.30 & 6.36e-7 & 4.88e-7\\ 
    28  &0& 1e-7 & 2,798 &   204.32 & {\bf 4.71e-7} &  {\bf 4.45e-7}\\ 
    \hline
    total time  &   \multicolumn{6}{c}  {\bf 2,693.09}\\
    \hline
\end{tabular}
    }
\end{table}

\begin{table}[]
\centering

    \caption{The SSG model with width 50 for learning function \eqref{non-differentiable function}}
    \label{table: single results 50 nondifferentiable}
\footnotesize{ 
    \begin{tabular}{c||c|c|c|c|c|c}
\hline
structure  & $\alpha$ &   $\epsilon$ &  epoch    & train time (second)& rse(train) & rse(test)\\
  \hline
SSG-1 &  1e-3& 1e-7 & 1,999 & 695.54 & 1.45e-4 & 1.14e-4\\
SSG-1 &  1e-3& 1e-7 & 4,999 & 1,753.83 & 7.82e-5 & 7.42e-5\\
SSG-1 &  1e-3& 1e-7 & 6,999 & 2,459.353 & 1.27e-5 & 1.20e-5\\
SSG-1 &  1e-3& 1e-7 & 9,999 & 3,532.53 & 1.34e-5 & 1.30e-5\\
\hline
SSG-2 &  1e-3& 1e-7 & 1,999 & 1,090.30 & 1.31e-4 & 1.15e-4\\
SSG-2 &  1e-3& 1e-7 & 4,999 & 2,756.39 & 1.11e-5 & 9.72e-6\\
SSG-2 &  1e-3& 1e-7 & 6,999 &  3,868.43 & 7.18e-5 & 7.23e-5\\
SSG-2 &  1e-3& 1e-7 & 9,999 & 5,560.13 & 2.48e-6 & 2.19e-6\\
\hline
SSG-3 &  1e-3& 1e-7 & 1,999 & 1,526.43 & 3.83e-5 & 3.82e-5\\
SSG-3 &  1e-3& 1e-7 & 4,999 &  3,855.37 & 5.11e-5 & 5.44e-5\\
SSG-3 &  1e-3& 1e-7 & 6,999 & 5,405.37 & 7.01e-5 & 6.86e-5\\
SSG-3 &  1e-3& 1e-7 & 9,999 & 7,779.87 & 3.55e-6 & 3.47e-6\\
\hline
SSG-4 &  1e-3& 1e-7 & 1,999 & 1,264.28 & 1.17e-4 & 1.21e-4\\
SSG-4 &  1e-3& 1e-7 & 4,999 &  3,216.60 & 2.76e-4 & 2.65e-4\\
SSG-4 &  1e-3& 1e-7 & 6,999 & 4,530.62 & 6.05e-6 & 5.75e-6\\
SSG-4 &  1e-3& 1e-7 & 9,999 & 6,528.63 & 8.31e-6 & 8.41e-6\\
\hline
SSG-5 &  1e-3& 1e-7 & 1,999 & 1,425.75 & 7.52e-5 & 7.77e-5\\
SSG-5 &  1e-3& 1e-7 & 4,999 &  3,599.71 & 2.42e-6 & 2.48e-6\\
SSG-5 &  1e-3& 1e-7 & 6,999 & 5,066.60 & 3.05e-5 & 3.14e-5\\
SSG-5 &  1e-3& 1e-7 & 9,999 & 7,310.38 & 1.33e-5 & 1.46e-5\\
\hline
    \end{tabular}
      }
\end{table}

\begin{table}[]
    \centering
  \caption{The SSG model with width 100 for learning function \eqref{non-differentiable function}}
    \label{table: single results 100 nondifferentiable} 
\footnotesize{

    \begin{tabular}{c||c|c|c|c|c|c}
\hline
structure  & $\alpha$ &   $\epsilon$ &  epoch    & train time (second)& rse(train)  & rse(test)\\
\hline
SSG-6 &  1e-3& 1e-7 & 1,999 & 1,294.46 & 3.35e-5 & 3.53e-5\\
SSG-6 &  1e-3& 1e-7 & 4,999 & 3,313.18 & 5.81e-6 & 5.71e-6\\
 SSG-6 &  1e-3& 1e-7 & 6,999 & 4,661.39 & 3.77e-6 & 4.06e-6\\
SSG-6 &  1e-3& 1e-7 & 9,999 & 6,779.26 & 2.07e-6 & 2.19e-6\\
 \hline
SSG-7 &  1e-3& 1e-7 & 1,999 & 2,110.72 & 3.47e-5 & 3.84e-5\\
SSG-7 &  1e-3& 1e-7 & 4,999 & 5,426.27 & 3.44e-6 & 3.40e-6\\
SSG-7 &  1e-3& 1e-7 & 6,999 &  7,649.76 & 6.84e-6 & 5.95e-6\\
SSG-7 &  1e-3& 1e-7 & 9,999 & 11,130.31 & 1.22e-6 & 1.15e-6\\
\hline
SSG-8 &  1e-3& 1e-7 & 1,999 & 2,969.82 & 1.52e-4 & 1.57e-4\\
SSG-8 &  1e-3& 1e-7 & 4,999 &  7,578.61 & 1.39e-4 & 1.39e-4\\
SSG-8 &  1e-3& 1e-7 & 6,999 & 10,704.66 & 3.86e-6 & 3.63e-6\\
SSG-8 &  1e-3& 1e-7 & 9,999 & 15,560.65 & 1.47e-6 & 1.52e-6\\
\hline
SSG-9 &  1e-3& 1e-7 & 1,999 & 3,717.08 & 1.05e-4 & 1.01e-4\\
SSG-9&  1e-3& 1e-7 & 4,999 &  9,495.76 & 1.16e-5 & 1.83e-5\\
SSG-9 &  1e-3& 1e-7 & 6,999 & 13,385.62 & 4.37e-6 & 4.10e-6\\
SSG-9 &  1e-3& 1e-7 & 9,999 & 19,467.99 & 4.51e-6 & 4.27e-6\\
\hline
SSG-10 &  1e-3& 1e-7 & 1,999 & 4,891.87 & 5.82e-5 & 5.72e-5\\
SSG-10 &  1e-3& 1e-7 & 4,999 &  12,507.98 & 3.00e-6 & 2.73e-6\\
SSG-10 &  1e-3& 1e-7 & 6,999 & 17,688.91 & 2.26e-6 & 2.11e-6\\
SSG-10 &  1e-3& 1e-7 & 9,999 & 25,676.95 & 5.88e-7 & 5.80e-7\\
\hline
    \end{tabular}
     }
\end{table}

\begin{table}[]
    \centering
  \caption{The SSG model with width 200 for learning function \eqref{non-differentiable function}}
    \label{table: single results 200 nondifferentiable} 
\footnotesize{
   
    \begin{tabular}{c||c|c|c|c|c|c}
\hline
structure  & $\alpha$ &   $\epsilon$ &  epoch    & train time (second)& rse(train)  & rse(test)\\
  \hline
SSG-11 &  1e-3& 1e-7 & 1,999 & 3,635.98 & 8.58e-6 & 8.64e-6\\
SSG-11 &  1e-3& 1e-7 & 4,999 & 9,278.44 & 8.22e-5 & 7.70e-5\\
SSG-11 &  1e-3& 1e-7 & 6,999 & 13,070.06 & 3.20e-6 & 3.03e-6\\
SSG-11 &  1e-3& 1e-7 & 9,999 & 19,161.92 & 2.19e-5 & 2.09e-5\\
 \hline
SSG-12 &  1e-3& 1e-7 & 1,999 & 6,291.03 & 4.12e-5 & 3.96e-5\\
SSG-12 &  1e-3& 1e-7 & 4,999 & 16,292.61 & 4.11e-6 & 3.81e-6\\
SSG-12 &  1e-3& 1e-7 & 6,999 &  22,975.50 & 1.32e-5 & 1.27e-5\\
SSG-12 &  1e-3& 1e-7 & 9,999 & 33,725.46 & 3.76e-7 & 4.05e-7\\
\hline
SSG-13 &  1e-3& 1e-7 & 1,999 & 8,705.78 & 2.55e-5 & 2.53e-5\\
SSG-13&  1e-3& 1e-7 & 4,999 &  22,534.86 & 2.06e-5 & 2.02e-5\\
SSG-13 &  1e-3& 1e-7 & 6,999 & 31,749.15 & 5.87e-6 & 6.08-6\\
SSG-13 &  1e-3& 1e-7 & 9,999 & 46,641.80 & 3.90e-7 & 4.42e-7\\
\hline
SSG-14 &  1e-3& 1e-7 & 1,999 & 11,275.22 & 3.87e-5 & 3.66e-5\\
SSG-14 &  1e-3& 1e-7 & 4,999 &  28,947.24 & 2.03e-5 & 2.18e-5\\
SSG-14 &  1e-3& 1e-7 & 6,999 & 40,824.20 & 6.74e-6 & 6.62e-6\\
SSG-14&  1e-3& 1e-7 & 9,999 & 59,839.41 & 1.52e-6 & 1.63e-6\\
\hline
SSG-15&  1e-3& 1e-7 & 1,999 & 12,519.85 & 2.45e-5 & 2.54e-5\\
SSG-15 &  1e-3& 1e-7 & 4,999 &  32,096.41 & 3.27e-6 & 3.27e-6\\
SSG-15 &  1e-3& 1e-7 & 6,999 & 45,226.71 & 9.36e-6 & 8.67e-6\\
SSG-15 &  1e-3& 1e-7 & 9,999 & 66,309.12 & 2.99e-6 & 2.54e-6\\
\hline
    \end{tabular}
     }
\end{table}

\begin{table}[]
\centering
  \caption{The SSG model with width 300 for learning function \eqref{non-differentiable function}}
    \label{table: single results 300 nondifferentiable} 
\footnotesize{
\begin{tabular}{c||c|c|c|c|c|c}
\hline
structure  & $\alpha$ &   $\epsilon$ &  epoch    & train time (second)& rse(train)  & rse(test)\\
  \hline
SSG-16 &  1e-3& 1e-7 & 1,999 & 7,628.01 & 1.14e-5 & 1.10e-5\\
SSG-16 &  1e-3& 1e-7 & 4,999 & 19,262.90 & 5.06e-5 & 4.74e-5\\
SSG-16 &  1e-3& 1e-7 & 6,999 & 27,085.22 & 1.16e-5 & 1.21e-5\\
SSG-16 &  1e-3& 1e-7 & 9,999 & 39,964.45 & 5.56e-7 & 5.07e-7\\
 \hline
SSG-17 &  1e-3& 1e-7 & 1,999 & 13,244.32 & 2.73e-5 & 2.57e-5\\
SSG-17 &  1e-3& 1e-7 & 4,999 & 34,003.43 & 2.04e-4 & 2.20e-4\\
SSG-17 &  1e-3& 1e-7 & 6,999 &  22,975.50 & 1.32e-5 & 1.27e-5\\
SSG-17 &  1e-3& 1e-7 & 9,999 & 33,725.46 & 3.76e-7 & 4.05e-7\\
\hline
SSG-18 &  1e-3& 1e-7 & 1,999 & 8,705.78 & 2.55e-5 & 2.53e-5\\
SSG-18 &  1e-3& 1e-7 & 4,999 &  22,534.86 & 2.04e-4 & 2.20e-4\\
SSG-18 &  1e-3& 1e-7 & 6,999 & 47,986.34 & 4.88e-6 & 4.83-6\\
SSG-18 &  1e-3& 1e-7 & 9,999 & 70,782.28 & 5.59e-5 & 5.01e-5\\
\hline
SSG-19 &  1e-3& 1e-7 & 1,999 & 22,536.85 & 1.60e-6 & 1.44e-6\\
SSG-19 &  1e-3& 1e-7 & 4,999 & 57,989.43 & 2.06e-6 & 2.16e-6\\
SSG-19 &  1e-3& 1e-7 & 6,999 & 81,764.25 & 4.95e-6 & 5.00e-6\\
\hline
SSG-20 &  1e-3& 1e-7 & 1,999 & 26,002.99 & 1.09e-6 & 9.85e-7\\
SSG-20&  1e-3& 1e-7 & 4,999 &  65,254.11 & 9.56e-6 & 9.64e-6\\
SSG-20 &  1e-3& 1e-7 & 6,999 & 91,917.19 & 3.22e-6 & 3.57e-6\\
\hline
    \end{tabular}

    }
\end{table}

\begin{table}[]
\centering
  \caption{The SSG model with variable widths for learning function \eqref{non-differentiable function}}
    \label{table: single results SSG-21 nondifferentiable} 
\footnotesize{
\begin{tabular}{c||c|c|c|c|c|c}
\hline
structure  & $\alpha$ &   $\epsilon$ &  epoch    & train time (second)& rse(train)  & rse(test)\\
  \hline
SSG-21& 1e-3  & 1e-7& 1,999 & 63,083.11 & 1.20e-5 &1.11e-5\\
SSG-21& 1e-3  & 1e-7& 4,999 & 164,562.52 &  7.29e-6&6.77e-6\\
\hline
    \end{tabular}
    
    }
\end{table}

We report the numerical results in Tables \ref{table: example mul-dnn results}-\ref{table: single results SSG-21 nondifferentiable}. In the tables, $\epsilon$ is the stopping error for iterations. From Table \ref{table: example mul-dnn results}, we observe that the SAL model with structure SAL-1 generates an approximation with accuracy: res(train) = 8.19e-6 and res(test) = 9.01e-6 and total training time 491.24 seconds. While the SSG model with structure SSG-4 generates approximations with nearly comparable accuracy res(train) = 8.31-6, res(test) = 8.41-6, total training time 6,528.63 seconds (13.3 times as that of the SAL model), see Table \ref{table: single results 50 nondifferentiable}. From Table \ref{table: example mul-dnn results:SAL-2}, the SAL model with structure SAL-2 generates an approximation with accuracy: res(train) = 4.71e-7, res(test) = 4.45e-7 and total training time 2,693.09 seconds. While the SSG model with structure SSG-10, SSG-12, SSG-16 generate approximations with nearly comparable accuracy, respectively, res(train) = 5.88e-7, res(test) = 5.80e-7, total training time 25,676.95 seconds (9.5 times as that of the SAL model), res(train) = 3.76e-7, res(test) = 4.05e-7, total training time 33,725.46 seconds (12.5 times as that of the SAL model), and res(train) = 5.56e-7, res(test) = 5.07e-7, total training time 39,964.45 seconds (14.8 times as that of the SAL model). These numerical results reveals that with comparable accuracy for both training and test data, the SAL model outperforms the SSG model with various network structures in 9.5-14.8 times speedup. The SSG model with all other network structures does not produce results comparable to those produced by the SAL model.

From Tables \ref{table: example mul-dnn results} and \ref{table: example mul-dnn results:SAL-2}, we see that for the SAL model, the quadratic optimization problems for all grades can be efficiently solved by the Nesterov algorithm. The computing time for all grades is relatively small. For both network structures, the SAL model exhibits fast convergence. Moreover, as a new grade is added to the approximation, the errors for both training data and test data reduce. This confirms the theoretical results established in section 5.

\subsection{Learning an oscillatory vector-valued function}
In our second example, we consider learning the oscillatory vector-valued function
\begin{equation}\label{oscillatory function}
\ff(x) := \left(\psi_1(x), \psi_2(x), \ldots, \psi_{20}(x)\right)^{\top}, \ \ x\in [0,1],  
\end{equation}
where 
$$
\psi_k(x):= \left(a_k x^2 + b_k x + c_k\right)\sin\left(100x \right),    \ \ x \in [0, 1], \ \ k=1, 2, \dots, 20.
$$
The coefficients $a_k$ are chosen by $\textbf{a} = 5*np.random.randn(20)$, $b_k$ by $\textbf{b} = -5*np.random.randn(20)$, and $c_k$ by $\textbf{c} = 10*np.random.randn(20)$. 
To avoid randomness, we use $np.random.seed(1)$. In this example, $[a,b]:=[0,1]$, $\delta:=0$, $m:=5,000$, $m':=1,000$, and $t:=20$.

For the SAL model, we employ the network structure:

\noindent\textbf{SAL-3} composes of one input layer, 10 hidden layers of uniform width 300 and one
output layer.

For the SSG model, we consider 20 different network structures listed in Table \ref{Structures-for-SSG}.
For both the SAL and SSG models, we use $\frac{1}{2}\sin x+ \frac{1}{2}\cos x$ activation function for the first hidden layer, and use ReLU activation function for the remaining hidden layers, and identity activation for output layer.

We adopt two stopping
criteria for iterations for solving the optimization problems for the grades of the SAL model. Stopping criterion I is either the relative error between the
function values of two consecutive steps less than the given number $\epsilon$ or the iteration number equal to 10,000 for grade 1, to 20,000 for grades 2-4, to 30,000 for grades 5-8, and to 40,000 for grades 9-10. Stopping criterion II is either the relative error between the
function values of two consecutive steps less than the given number $\epsilon$ or the iteration number equal to 50,000 for all grades. The numbers of iterations reported in Tables \ref{table:SAL-3-I} and \ref{table:SAL-3-II} are the actual numbers used in the iterations. 
The parameters involved in the discrete smoothing operator \eqref{D-Smoothing} for the SAL model are chosen as $a_x:=x-6\tau$, $b_x:=x+6\tau$,  and $M:=200$
for this example. Here, $\tau$ varies from grade to grade and see Tables \ref{table:SAL-3-I} and \ref{table:SAL-3-II} for details.
The stopping criterion for the SSG model is the same as that in the first example.

Numerical results for this example are reported in Tables \ref{table:SAL-3-I}-\ref{table: single results 300}. These results show that the proposed SAL model outperforms the SSG model significantly. The SAL model with stopping criteria I and II generates, respectively, approximations with accuracy rse(train) = 6.13e-8, rse(test) = 5.77e-8, total training time 4,095.39 seconds, and rse(train) = 4.09e-9, rse(test) = 4.44e-9, total training time 6,231.79 seconds. While the SSG model with all network structures listed in Table \ref{Structures-for-SSG} and various stopping criteria cannot reach the approximation accuracy that the SAL model does. The best result produced by the SSG model is rse(trian) = 4.23e-5, rse(test) = 4.31e-5, with training time 22,219.42 second, when the network structure is chosen as SSG-12 with epoch 7,000, (see, Table \ref{table: single results 200} for details).

Once again, these numerical results show that the SAL model converges as the number of grades increases. The error reduction demonstrated in Tables \ref{table:SAL-3-I} and \ref{table:SAL-3-II} confirms the theoretical results established in section 5.

\begin{table}[]
 \centering
 \caption{The SAL model with structure SAL-3 with  stopping criterion I for learning function \eqref{oscillatory function}}
    \label{table:SAL-3-I}
\footnotesize{
    \begin{tabular}{c||c|c|c|c|c|c}
    \hline 
       grade  &   $\tau$ &    $\epsilon$ &  iteration \#    & train time (second)& rse(train)  & rse(test) \\
     \hline   
       1 &  0&    1e-6 & 870 & 17.37 & 9.98e-1 & 9.93e-1\\
       2 &  0 & 1e-6 & 19,999 & 229.08 & 6.21e-3 & 5.98e-3\\
       3 & 5e-3 &  1e-6 & 19,999 &   374.49 & 3.22e-3 & 3.12e-3\\
    
       4 & 4e-3&  1e-7 & 19,999 &   366.92 & 7.68e-3 &  7.34e-3\\
   
       5 & 3e-3& 1e-7 & 29,999 &  484.86 &1.51e-4 &  1.43e-4\\
     
    6 & 3e-3&  1e-7 & 29,999 &  473.58 & 3.14e-5 &  3.14e-5\\
     
    7 & 2e-3&  1e-7 & 29,999 &  476.97 & 3.97e-6 &  4.35e-6\\
   
    8 & 1e-3&  1e-7 & 29,999 &  482.11 & 5.03e-7 &  5.35e-7\\
    
     9 & 1e-3& 1e-7 & 39,999 &  599.95 & 1.45e-7 &  1.45e-7\\
       
    10 & 1e-3&  1e-7 & 39,999 &  590.06 & {\bf 6.13e-8} &  {\bf 5.77e-8}\\ 
    \hline
    total time  &   \multicolumn{6}{|c}  {\bf 4,095.39}\\
    \hline
    \end{tabular}
     }
\end{table}

\begin{table}[]

\footnotesize{
    \centering
        \caption{The SAL model with structure SAL-3 with stopping criterion II for learning function \eqref{oscillatory function}}
    \label{table:SAL-3-II}
    \begin{tabular}{c||c|c|c|c|c|c}
    \hline
       grade  &   $\tau$ &   $\epsilon$ &  iteration \#    & train time (second)& rse(train)  & rse(test) \\
    \hline   
       1 &  0& 1e-6 & 870 & 17.35 & 9.98e-1 & 9.93e-1\\
   
       2 &  0 &  1e-6 & 49,999 & 562.60 & 2.69e-3 & 2.50e-3\\
   
       3 &  5e-3 & 1e-6 & 49,999 &  705.81 & 5.96e-4 & 5.89e-4\\
   
       4 & 4e-3& 1e-7 & 49,999 & 710.07 & 1.06e-4 &  1.03e-4\\
      
       5 & 3e-3& 1e-7 & 49,999 &  693.21  &1.33e-5 &  1.30e-5\\
       
    6 & 3e-3&  1e-7 & 49,999 &  710.72 & 1.36e-6 &  1.41e-6\\
     
    7 & 2e-3&  1e-7 & 49,999 &  703.58 & 1.36e-7 &  1.37e-7\\
   
    8 & 1e-3& 1e-7 & 49,999 &  705.44 & 1.92e-8 &  1.97e-8\\
   
     9 & 1e-3&  1e-7 & 49,999 &  701.53 & 7.11e-9 &  7.48e-9\\
      
    10 & 1e-3& 1e-7 & 49,999 &  721.48 & {\bf 4.09e-9} &  {\bf 4.44e-9}\\
    \hline
    total time  &   \multicolumn{6}{|c}  {\bf 6,231.79}\\
    \hline  
    \end{tabular}

    }
\end{table}

\begin{table}[]
   
\footnotesize{

    \centering
     \caption{The SSG model with width 50 for learning function \eqref{oscillatory function}}
    \label{table: single results 50}
    \begin{tabular}{c||c|c|c|c|c|c}
    \hline 
structure  & $\alpha$ &   $\epsilon$ & epoch  & train time (second)& rse(train)  & rse(test) \\
  \hline
SSG-1 &  1e-4& 1e-7 & 1,999 & 500.11 & 8.46e-2 & 8.00e-2\\
SSG-1 &  1e-4& 1e-7 & 4,999 & 1,258.62 & 8.00e-2 & 7.45e-2\\
SSG-1  &  1e-4& 1e-7 & 6,999 & 1,765.27 & 2.90e-3 & 2.83e-3\\
SSG-1  &  1e-4& 1e-7 & 9,999 & 2,546.10 & 1.65e-3 & 1.60e-3\\
 \hline
SSG-2 &  1e-4& 1e-7 & 1,999 & 777.48 & 3.13e-3 & 3.18e-4\\
SSG-2  &  1e-4& 1e-7 & 4,999 & 1,964.83 & 5.76e-4 & 5.93e-4\\
SSG-2  &  1e-4& 1e-7 & 6,999 &  2,761.31 & 2.56e-3 & 2.60e-3\\
SSG-2 &  1e-4& 1e-7 & 9,999 & 3,969.41 & 5.83e-4 & 5.98e-4\\
\hline
SSG-3 &  1e-4& 1e-7 & 1,999 & 1,069.02 & 4.47e-4 & 3.81e-4\\
SSG-3  &  1e-4& 1e-7 & 4,999 &  2,708.85 & 1.79e-4 & 1.64e-4\\
SSG-3  &  1e-4& 1e-7 & 6,999 & 3,804.06 & 8.00e-5 & 7.43e-5\\
SSG-3  &  1e-4& 1e-7 & 9,999 & 5,485.55 & 2.00e-4 & 2.08e-4\\
\hline
SSG-4  &  1e-4& 1e-7 & 1,999 & 1,350.66 & 3.05e-3 & 3.08e-3\\
SSG-4  &  1e-4& 1e-7 & 4,999 &  3,437.56 & 9.03e-5 & 8.66e-5\\
SSG-4 &  1e-4& 1e-7 & 6,999 & 4,844.77 & 9.70e-4 & 9.74e-4\\
SSG-4  &  1e-4& 1e-7 & 9,999 & 6,981.31 & 5.49e-3 & 5.33e-3\\
\hline
SSG-5 &  1e-4& 1e-7 & 1,999 & 1,484.32 & 9.82e-4 & 9.44e-4\\
SSG-5  &  1e-4& 1e-7 & 4,999 &  3,745.63 & 1.63e-4 & 1.63e-4\\
SSG-5  &  1e-4& 1e-7 & 6,999 & 5,278.43 & 7.96e-5 & 8.43e-5\\
SSG-5  &  1e-4& 1e-7 & 9,999 & 7,638.00 & 8.06e-4 & 8.67e-4\\
\hline
    \end{tabular}

    }
\end{table}

\begin{table}[]
 \centering
 \caption{The SSG model with width 100 for learning function \eqref{oscillatory function}}
    \label{table: single results 100}
\footnotesize{
   
    \begin{tabular}{c||c|c|c|c|c|c}
    \hline 
 structure  & $\alpha$ &   $\epsilon$ &  epoch    & train time (second)& rse(train)  & rse(test)\\
  \hline
SSG-6 &  1e-4& 1e-7 & 1,999 & 1,326.71 & 2.92e-3 & 3.03e-3\\
 
SSG-6 &  1e-4& 1e-7 & 4,999 & 3,425.84 & 2.83e-3 & 2.64e-3\\

SSG-6 &  1e-4& 1e-7 & 6,999 & 4,832.23 & 1.92e-3 & 1.88e-3\\

SSG-6 &  1e-4& 1e-7 & 9,999 & 7,005.57 & 2.45e-3 & 2.36e-3\\
 \hline
SSG-7 &  1e-4& 1e-7 & 1,999 & 1,918.29 & 2.42e-4 & 2.46e-4\\

SSG-7 &  1e-4& 1e-7 & 4,999 & 4,933.36 & 3.55e-4 & 3.70e-4\\
SSG-7 &  1e-4& 1e-7 & 6,999 &  6,993.76 & 1.93e-4 & 1.97e-4\\
SSG-7 &  1e-4& 1e-7 & 9,999 &  10,228.32 & 1.22e-3 & 1.23e-3\\
\hline
SSG-8 &  1e-4& 1e-7 & 1,999 & 2,664.25 & 2.59e-3 & 2.50e-3\\
SSG-8 &  1e-4& 1e-7 & 4,999 & 6,909.61 & 7.13e-5 & 7.28e-5\\
SSG-8 &  1e-4& 1e-7 & 6,999 &  9,848.36 & 3.77e-4 & 3.87e-4\\
SSG-8 &  1e-4& 1e-7 & 9,999 &  14,396.86 & 3.78e-4 & 3.77e-4\\
\hline
SSG-9 &  1e-4& 1e-7 & 1,999 &  3,436.43 & 9.64e-4 & 9.91e-4\\
SSG-9  &  1e-4& 1e-7 & 4,999 & 8,915.12 & 1.18e-3 & 1.17e-3\\
SSG-9  &  1e-4& 1e-7 & 6,999 & 12,705.80 & 5.18e-5 & 4.99e-5\\
SSG-9  &  1e-4& 1e-7 & 9,999 & 18,632.12 & 8.13e-4 & 9.03e-4\\
\hline
SSG-10 &  1e-4& 1e-7 & 1,999 & 4,083.72 & 1.54e-4 & 1.54e-4\\
SSG-10  &  1e-4& 1e-7 & 4,999 &  10,468.17 & 9.97e-4 &  9.18e-4\\
SSG-10  &  1e-4& 1e-7 & 6,999 & 14,824.08 & 5.06e-5 & 5.00e-5\\
SSG-10  &  1e-4& 1e-7 & 9,999 & 14,824.08 & 1.12e-3 & 1.22e-3\\
\hline
\end{tabular}
}
\end{table}

\begin{table}[]
    \centering
\caption{The SSG model with width 200 for learning function \eqref{oscillatory function}}
    \label{table: single results 200}
\footnotesize{
\begin{tabular}{c||c|c|c|c|c|c}
\hline 
structure  & $\alpha$ &   $\epsilon$ &  epoch    & train time (second)& rse(train)  & rse(test)\\
  \hline
SSG-11 &  1e-4& 1e-7 & 1,999 & 3,092.26 & 2.96e-3 & 2.90e-3\\
SSG-11  &  1e-4& 1e-7 & 4,999 & 8,172.53 & 8.14e-4 & 8.70e-4\\
SSG-11  &  1e-4& 1e-7 & 6,999 & 11,689.40 & 5.57e-4 & 5.74e-4\\
SSG-11  &  1e-4& 1e-7 & 9,999 & 17,307.35 & 6.24e-4 & 6.24e-4\\
 \hline
SSG-12 &  1e-4& 1e-7 & 1,999 & 5,847.95 & 5.25e-3 & 5.56e-3\\
SSG-12  &  1e-4& 1e-7 & 4,999 & 15,545.16 & 3.16e-3 & 3.30e-3\\
SSG-12  &  1e-4& 1e-7 & 6,999 &  22,219.42 & 4.23e-5 & 4.31e-5\\
SSG-12  &  1e-4& 1e-7 & 9,999 &  32,896.63 & 4.17e-4 & 4.47e-4\\
\hline
SSG-13 &  1e-4& 1e-7 & 1,999 &  8,663.59 & 2.55e-3 & 2.45e-3\\
SSG-13  &  1e-4& 1e-7 & 4,999 &  22,843.51 & 7.43e-5 & 7.76e-5\\
SSG-13  &  1e-4& 1e-7 & 6,999 &  32,676.66 & 3.18e-4 & 3.25e-4\\
\hline
SSG-14 &  1e-4& 1e-7 & 1,999 &  11,837.76 & 4.02e-2 &  4.28e-2\\
SSG-14 &  1e-4& 1e-7 & 4,999 &  31,609.42 & 8.92e-5 &  8.98e-5\\
SSG-14  &  1e-4& 1e-7 & 6,999 &  45,064.18 & 5.44e-4 &  5.34e-4\\
\hline
SSG-15 &  1e-4& 1e-7 & 1,999 & 12,887.15 & 3.88e-3 & 3.86e-3\\
SSG-15  &  1e-4& 1e-7 & 4,999 &  34,594.90 & 8.65e-5 & 8.64e-5\\
SSG-15 &  1e-4& 1e-7 & 6,999 &  49,549.28 & 2.87e-3 & 2.93e-3\\
\hline
\end{tabular}
    }
\end{table}

\begin{table}[]
 \caption{The SSG model with width 300 for learning function \eqref{oscillatory function}}
    \label{table: single results 300}
\footnotesize{
    \centering
    \begin{tabular}{c||c|c|c|c|c|c}
\hline
 structure  & $\alpha$ &   $\epsilon$ &  epoch    & train time (second)& rse(train)  & rse(test) \\
  \hline
SSG-16 &  1e-4& 1e-7 & 1,999 & 7,135.19 & 2.26e-3 & 2.31e-3\\
SSG-16  &  1e-4& 1e-7 & 4,999 & 18,729.82 & 7.75e-5 & 7.33e-5\\
SSG-16  &  1e-4& 1e-7 & 6,999 & 26,690.13 & 3.88e-4 & 4.10e-4\\
SSG-16  &  1e-4& 1e-7 & 9,999 & 39,600.80 & 2.59e-3 & 2.54e-3\\
\hline
SSG-17 &  1e-4& 1e-7 & 1,999 & 17,481.68 & 4.21e-4 & 4.19e-4\\
SSG-17  &  1e-4& 1e-7 & 4,999 & 46,000.92 & 1.04e-4 & 9.79e-5\\
\hline
SSG-18  &  1e-4& 1e-7 & 1,999 &  20,020.21 & 6.11e-3 & 6.03e-3\\
SSG-18  &  1e-4& 1e-7 & 4,999 &  52,578.57 & 1.20e-4 & 1.22e-4\\
\hline
SSG-19 &  1e-4& 1e-7 & 1,999 &  25,004.28 & 3.56e-2 &  2.46e-2\\
SSG-19 &  1e-4& 1e-7 & 4,999 &  66,193.35 & 4.88e-3 &  5.04e-3\\
\hline
SSG-20 &  1e-4& 1e-7 & 1,999 & 27,636.96 & 8.59e-3 & 8.48e-3\\
SSG-20 &  1e-4& 1e-7 & 4,999 & 73,692.73 & 1.02e-4 & 1.00e-4\\

\hline
\end{tabular}
   
    }
\end{table}

\section{Conclusive Remarks}

We have developed the SAL model to learn affine maps that define a DNN. Unlike the traditional deep learning model which solves one non-convex optimization problem to determine weight matrices and bias vector, the SAL model successively solves a sequence of {\it quadratic/convex} optimization problems, each of which defines one layer of a DNN. The proposed SAL model overcomes difficulties of the traditional deep learning model in solving a highly non-convex optimization problem with a large number of parameters for a DNN. The neural networks generated from the SAL model form an adaptive orthogonal basis, for a given function, which enjoys both the Pythagorean identity and the Parseval identity as the Fourier basis does. We further show the convergence result of the SAL model without pooling: Either the SAL process terminates after a finite number of grades or the optimal error function of a grade reduces in norm strictly from that of the previous grade toward a limit. Two proof-of-concept numerical examples presented in the paper demonstrate that the proposed SAL model outperforms significantly the standard single-grade deep learning model in training
time, training accuracy and prediction accuracy. 
Adoption of the SAL model to solving practical problems requires further investigation.
\bigskip

\noindent {\bf Acknowledgement:}
The author is indebted to graduate student Mr. Ronglong Fang for his assistance in coding for the numerical examples presented in section 8. This work is supported in part by US National Science Foundation under grants DMS-1912958  and DMS-2208386, and by  US National Institutes of Health under grant R21CA263876. 

\noindent {\bf Conflict of Interest Statement:} 
The author declared that there is no conflict of interest.

\end{document}